\documentclass[journal]{IEEEtran}

\usepackage{url}
\usepackage{graphicx}
\usepackage{xcolor}

\usepackage{amsmath,amsfonts,amssymb,amsthm,mathtools}
\usepackage{bm,bbm}
\usepackage{mathrsfs}
\usepackage{empheq}
\usepackage{enumerate}

\newtheorem{theorem}{Theorem}
\newtheorem{proposition}{Proposition}
\newtheorem{lemma}{Lemma}
\newtheorem{corollary}{Corollary}

\newtheorem{remark}{Remark}

\usepackage{nameref}
\usepackage[capitalize]{cleveref}
\crefname{equation}{equation}{equations}
\crefname{assumption}{assumption}{assumptions}
\creflabelformat{enumi}{(#2#1#3)}

\newcommand{\sgn}{\mathrm{sgn}}

\global\long\def\Reals{\mathbb{R}}

\global\long\def\intd{\mathrm{d}} % differential d

\DeclareMathOperator{\sign}{sgn}

\newcommand{\EE}{\mathbb{E}}

\newcommand{\PP}{\mathbb{P}}

\newcommand{\KL}{\mathsf{KL}}

\usepackage{mathtools}

\providecommand\given{} % so it exists
\newcommand\SetSymbol[1][]{
	\nonscript\,#1:\nonscript\,\mathopen{}\allowbreak}
\DeclarePairedDelimiterX\Set[1]{\lbrace}{\rbrace}%
{ \renewcommand\given{\SetSymbol[]} #1 }
\DeclarePairedDelimiter{\Inner}{\langle}{\rangle}

\DeclareMathOperator*{\argmin}{arg\,min}

\DeclarePairedDelimiter{\braces}{\{}{\}}
\DeclarePairedDelimiter{\ip}{\langle}{\rangle}
\DeclarePairedDelimiter{\brackets}{(}{)}
\DeclarePairedDelimiter{\sqbrackets}{[}{]}
\DeclarePairedDelimiter{\norm}{\lVert}{\rVert}
\DeclarePairedDelimiter{\abs}{\lvert}{\rvert}
\DeclarePairedDelimiter{\floor}{\lfloor}{\rfloor}
\DeclarePairedDelimiter{\ceil}{\lceil}{\rceil}

\newcommand{\RR}{\mathbb{R}}
\newcommand{\NN}{\mathbb{N}}
\newcommand{\II}{\mathbbm{1}}

\begin{document}

\title{Fundamental limits for learning hidden Markov model parameters}

\author{Kweku~Abraham, Elisabeth~Gassiat and Zacharie~Naulet}

\markboth{Journal of \LaTeX\ Class Files,~Vol.~14, No.~8, August~2021}%
{Shell \MakeLowercase{\textit{et al.}}: A Sample Article Using IEEEtran.cls for IEEE Journals}


\maketitle

\begin{abstract}
  We study the frontier between learnable and unlearnable hidden Markov models
  (HMMs). HMMs are flexible tools for clustering dependent data coming from
  unknown populations. The model parameters are known to be fully identifiable
  (up to label-switching) without any modelling assumption on the distributions
  of the populations as soon as the clusters are distinct and the hidden chain
  is ergodic with a full rank transition matrix. In the limit as any one of
  these conditions fails, it becomes impossible in general to identify
  parameters. For a chain with two hidden states we prove nonasymptotic minimax
  upper and lower bounds, matching up to constants, which exhibit thresholds at
  which the parameters become learnable. We also provide an upper bound on the
  relative entropy rate for parameters in a neighbourhood of the unlearnable
  region which may have interest in itself.
\end{abstract}


% \fi
\IEEEpeerreviewmaketitle

\footnotetext{\textcopyright 2022 IEEE.  Personal use of this material is permitted.  Permission from IEEE must be obtained for all other uses, in any current or future media, including reprinting/republishing this material for advertising or promotional purposes, creating new collective works, for resale or redistribution to servers or lists, or reuse of any copyrighted component of this work in other works.}

	\section{Introduction}
	\label{sec:introduction}

	\subsection{Context and motivation}
	Finite state space hidden Markov models (HMMs) are widely used in applications to model observations coming from different populations. HMMs can be viewed as particular mixture models. In the latter, given a latent sequence of cluster labels $(X_n)_{n\in \NN}$ taking values in a finite set, the observed data $(Y_n)_{n\in\NN}$ is a sequence of independent random variables with, for each $n$, the distribution of $Y_n$ depending only on $X_n$. When the $X_n$ are independent, a mixture model is not identifiable: various convex combinations of population probability distributions can lead to the same distribution for the observations. This is true even for  observations taking values in a finite alphabet: one cannot recover two different multinomial distributions from a convex combination of them.

	For a HMM, one adds the extra structure that $(X_n)_{n\in\NN}$ forms a Markov chain. In sharp %stark
	contrast to the independent setting, with hidden Markov structure one can recover the distribution of data for each population absent virtually any constraint on these distributions (known in this context as the emission distributions). This fact had been observed in applied papers, and a theoretical proof that parameters can be identified with minimal assumptions is relatively recent, given for HMMs taking values in a finite set in \cite{MR2549554, MR2926144, AGHKT14} and extended to allow for emission distributions modelled nonparametrically (but still with the underlying Markov chain having finite state space) in \cite{MR3439359, MR3509896}. HMMs therefore form a tractable class of models nevertheless rich enough to model many practical clustering settings well: see for instance \cite{CC00, Lef03, LWM03, SC09, MR2797735, VBM13}. In this context note that given good estimates of the model parameters one can almost match the optimal clustering and testing behaviour of the Bayes classifier (e.g.\ see \cite{SunCai09}, \cite{EIK:2021}); let us emphasise once more that this is possible essentially absent any constraint on the emission distributions, in contrast to typical clustering algorithms which may require parametric modelling or separation of clusters.

	In drawing a contrast between the independent and dependent cases, we have so far omitted to mention that of course an independent model is a degenerate subcase of a Markov model. There are three ways in which the data $(Y_n)_{n\in\NN}$ can fail to exhibit dependence: when the population labels themselves are in reality independently distributed; when the emission distributions are identical; or when only one population is observed. 
	Without extra modelling assumption on the populations distributions, learnability of the parameters with a finite number of observations becomes  difficult near the independent case. This occurs when one of the populations is sparse, which is a typical setting for multiple testing.  It  also occurs  when the two populations have close distributions; knowing how far separated the populations have to be for clustering to be possible without further structural assumptions is of interest. Finally, this occurs  when the cluster labels have very weak dependence.
It has also been observed empirically in \cite{Phys:2020} that the EM algorithm can exhibit bad behavior in some regions of the parameter space. 
 It is thus of theoretical and practical importance to understand quantitatively what happens when these limiting situations are approached in the setting of finitely many observations of the HMM sequence.

	The present work initiates an exploration of the limits of learnability of the hidden Markov parameters as the independent subcase is approached. We focus on the setting of two hidden states and multinomial data, and exhibit principles which should generalise to much wider settings.

	\subsection{Contribution}

	Our main result, \cref{cor:rates-for-theta}, gives upper and lower bounds showing the minimax estimation rate for the model parameters, exhibiting that these parameters can be learned if and only if the sample size $n$ is large enough compared to a suitable measure of the closeness of the data to the independent subcase.

	Important steps to get the main result are as follows.
We introduce a	reparametrisation of the model leading to a statistical distance which appears to be a key tool for the understanding of the fundamental limits of learning the HMM parameters near the independent subcase. This statistical distance is proved in \cref{pro:local-equivalence}
to be equivalent to the distance between the  distribution of three consecutive observations, and  leads to an explicit upper bound of the relative entropy rate for a specific part of parameters domain, see
	\cref{prop:kul}, which we believe could have interest in itself. Upper bounds for the learning of the new parameters are proved in \cref{thm:minimax-upper-bound} while (almost) matching lower bounds are proved in \cref{thm:mlb}.

	\subsection{Related work}

	Theoretical justification of a range of learning methods for HMMs with emission distributions modelled parametrically or nonparametrically have been developed in recent years: moment and tensor methods in \cite{AGHKT14, dCGlC17}, and model selection using penalized least squares estimation in \cite{MR3543517, MR3892326}, using penalized likelihood methods in \cite{Luc:2018}, or using other techniques in \cite{MR3862446}.
	These works all give both asymptotic and nonasymptotic upper bounds controlling the distance between estimators and the unknown parameters. All require the data to truly be dependent, but none quantify explicitly how their sample complexity results depend on the ``distance'' to independence.
	Indeed, quantifying this dependence requires a sharp understanding of how the distances between distributions evolve with respect to the distances between parameters, as done for particular parametric finite mixture models in \cite{MR3130325, MR3851757, Wu:2020}.

	Results in \cite{dCGlC17} control the
	propagation of errors from parameter estimation to the posterior probabilities
	when calculating the latter via plug-in, implying that good control on the risk
	of the estimators will ensure the performance of the empirical Bayes classifier
	is close to that of the true Bayes classifier (whose optimality for clustering is a standard
	result in decision theory \cite{devroye2013probabilistic}).

	A topic closely related to binary classification/clustering is multiple testing, in which one aims to identify within some large data set a collection of data points which come from a ``discovery'' hypothesis, rather than from the conservative null hypothesis. 	In this setting control of the false discovery rate has been obtained recently for a knockoffs-based method  in \cite{MR3912377} and for an empirical Bayes method in \cite{SunCai09,EIK:2021}; in each case estimation of the HMM parameters is an essential first step. Modelling the proportion of non-null signals as vanishingly small, as our results permit, would allow for further links to the setting of \emph{sparse} multiple testing, considered for example (with independent data) in \cite{MR2281879,MR4152112}.

	Relative entropy rate, or equivalently Kullback--Leibler rate, between  HMMs can be expressed using Blackwell's invariant measure \cite{B57}, but no explicit formulation exists 	\cite{SiNa2008}. Providing useful or meaningful upper and lower bounds is a subject of ongoing research
	\cite{Do03, LuoGuo09, FuhMei2015}. In \cref{prop:kul} we obtain a new bound on the Kullback-Leibler rate between HMMs which, compared to the aforementioned works, does a better job at capturing the effects of the underlying Markov dependency structure, at the expense of holding only for a restricted subspace of parameters.

	To the best of our knowledge no prior theoretical result exists addressing the learning of parameters of a HMM when approaching the independent case.
	By experimentally studying the EM algorithm when the multinomial emission distributions approach each other, the authors in \cite{Phys:2020} found a range of parameters for which the EM algorithm behaves badly. Their results provided some of the inspiration for the work herein, which we believe shows that such behaviour is primarily
a result of the investigated region approaching the limit where the parameters become unlearnable, not of a limitation of the EM algorithm specifically. %we believe such behaviour is reflective not of a limitation of the EM algorithm specifically, but of the fact that the investigated region approaches the limit where the parameters become unlearnable.

	Finally, let us mention that departure from the independence assumption has been noted to allow for better learning also in HMM settings free from the assumption that the Markov chain has a finite state space \cite{MR4119183, MR3433429} (at the expense of stricter assumptions on the emission distributions), and also in other problems including dynamic networks \cite{MR3689311, MR3964859},
	image denoising \cite{Ollion:2021}, and deconvolution \cite{gassiat:lecorff:lehericy:2020}.

	\subsection{Organisation of the paper}

We describe the setting in \cref{sec:setting} and state our main result in \cref{sec:main-results}. The key reparametrisation is given in \cref{sec:change-param} where we state  the basic propositions involving the statistical distance we define. Intermediate upper bound results are given in \cref{sec:upper-bounds} while lower bounds are in \cref{sec:lower-bounds}.
In \cref{sec:conclusion} we discuss our results and possible further work.
All proofs are deferred to \cref{sec:proofs}.

\subsection{Notation}
We write $\norm{f}=\ip{f,f}^{1/2}$  for the usual Euclidean norm and inner product. We write $a\vee b\coloneqq \max(a,b)$ and $a\wedge b\coloneqq \min(a,b)$. We write
$\KL$ for the Kullback--Leibler divergence between densities $p,q$ or between the corresponding distribution $P,Q$, $\KL(p,q)\equiv \KL(P,Q) = \EE_p(\log(p/q))$.

%\textcolor{red}{Define also $\sgn$, $\lesssim$ and $\gtrsim$?}

%\textcolor{red}{[zn: I suggest we add this section to define a couple of notations that were undefined before... in particular I use $\wedge$ and $\vee$ to replace $\min$ and $\max$ to help with the 2 column format... there are probably other things to explicit.]}

	\section{Setting}
	\label{sec:setting}

	Consider a two-state HMM with multinomial emissions, in which we observe the first $n$ entries of a sequence $\bm{Y}=(Y_1,Y_2,\dots)\in \braces{1,\dots,K}^\NN$ which, under a parameter $\theta=(p,q,f_0,f_1)$, satisfies
	\begin{equation}\label{eqn:model}
		\begin{split}
			\PP_\theta (Y_n=k \mid \bm{X}) &= f_{X_n}(k),\\
			\bm{X}=(X_n)_{n\in \NN} &\sim \operatorname{Markov}(\pi,Q),
		\end{split}
	\end{equation}
	with the $Y_j,~j\in\NN$ conditionally independent given $\bm{X}$.
	The vector $\bm{X}$ of `hidden states' takes values in $\braces{0,1}^\NN$ and the transition matrix of the
	chain is given by
	\begin{equation}
		%\label{eq:1}
		Q%
		\coloneqq%
		\begin{pmatrix}
			1 - p & p\\
			q & 1 - q
		\end{pmatrix}%
		,
	\end{equation}
	with the convention that for $j\geq 1$,
	$\PP_{\theta}(X_{j+1} = 0 \mid X_j = 0) = 1-p<1$ and
	$\PP_{\theta}(X_{j+1} = 0 \mid X_j = 1) = q>0$. The densities $f_0,f_1$ are the `emission densities' with respect to counting measure on $\braces{1,\dots,K}$.  Note that any function $g$ on $\braces{1,\dots,K}$ (such as $f_0$ and $f_1$)  may be identified with a vector $(g(a))_{1\leq a \leq K}$ in $\Reals^K$.
	 Grant also that
	$X_1$ is drawn from the stationary distribution of the chain, i.e.\
	$\PP_{\theta}(X_1 = 1) = p/(p + q)$. We throughout use $\PP_{\theta}$ to denote the law of $(\bm{X},\bm{Y})$, and all induced marginal and conditional laws.

	\smallskip

	In the limit where the sequence $\bm{Y}$ becomes independent and identically distributed (i.i.d.), learning the parameters becomes impossible due to standard identifiability issues for mixture models: the distribution of $Y_1$ may be decomposed in many ways as a convex combination of multinomials. This i.i.d.\ limit can be approached in three ways:
	\begin{enumerate}
		\item $p \approx 0$ or $q \approx 0$, and thus the chain
		$\bm{X}$ passes long periods of time in one of the two states;% remains in the same state for large periods of time;
		\item the transition matrix $Q$ is nearly singular, so that $\bm{X}$ itself is almost i.i.d; this is the case if $\abs{1 - p - q} \approx
		0$;
		\item the emission distributions are close to each other: $\norm{f_0-f_1}\approx 0$, where $\norm{\cdot}$ denotes the usual Euclidean norm, $\norm{f}^2=\sum \abs{f(k)}^2$.
	\end{enumerate}
	We adopt a minimax point of view and encapsulate all the above scenarios within
	the class of parameters defined, for some $\delta,\epsilon\in (0,1)$ and some $\zeta>0$, by
	\begin{equation*}% \label{eqn:Theta}
		\Theta=\Theta (\delta,\epsilon, \zeta) =
		\braces{ \theta\;:  p,q \geq\delta,\;\abs{1-p - q} \geq \epsilon,\:\norm{f_{0}-f_{1}} \geq \zeta
		}.
	\end{equation*}
	Introduce also the subset
	\begin{equation*}%\label{eqn:ThetaL}
		\Theta_{L} = \Theta_{L}(\delta,\epsilon,\zeta) = \Theta\cap \braces{ 1-\abs{1-p-q}\geq L}.
	\end{equation*}

	\begin{remark}
		Note that $1-\abs{1-p-q}$ is the absolute spectral gap of the chain $\bm{X}$, and hence the mixing time of the chain can be upper bounded uniformly in $\Theta_L$ since the state space has size 2 (so the chain is automatically reversible).
		Here $L$ may be arbitrarily small but we think of it as \emph{fixed}, in contrast to $\delta$, $\epsilon$ and $\zeta$ which are allowed to depend on $n$.
		With the introduction of this lower bound we still allow \emph{one} of $p,q$ to be vanishingly small (or arbitrarily close -- even equal -- to 1), but not both. \end{remark}

%	\begin{remark}
		%If $\zeta$ is too large compared to $1/K$ \textcolor{blue}{(recall that $K$ is the cardinality of the set of possible observations, so that $1/K$ is the common individual mass for the uniform distribution)}, $\Theta(\delta,\epsilon,\zeta)$ may be too small to be an interesting parameter space.
%		 To %avoid this and 
%		 ensure that $\Theta(\delta,\epsilon,\zeta)$ contains near uniform density pairs (see Lemma \ref{lem:exists-psi}), we assume a mild compatibility condition: that
%		\begin{equation} \label{eqn:compatibility} \zeta \leq \frac{\sqrt{2\floor{K/2}}}{4K}.\end{equation}
%		Recall that $K$ is the cardinality of the set of possible observations, so that $1/K$ is the common individual mass for the uniform distribution
%	\end{remark}

\section{Main results}\label{sec:main-results}

To avoid a label-switching issue discussed in the next section we assume that $f_0-f_1$ lies in some specified half-plane. Our main result is the following. The estimator $\hat{\theta}$ is built via plug-in from those constructed later in \cref{thm:minimax-upper-bound}.

\begin{theorem}
  \label{cor:rates-for-theta}
  There exist an estimator $\hat{\theta}=(\hat{p},\hat{q},\hat{f}_0,\hat{f}_1)$ and a constant $C=C(K,L)>0$ such that for all $1\leq x^2\leq n\delta^2\epsilon^4\zeta^6$,
  \begin{gather*}
    \sup_{\theta \in \Theta_L} \PP_\theta\brackets[\Big]{\abs{\hat{p}-p} \vee \abs{\hat{q}-q} >\frac{Cx (\delta \vee \epsilon \zeta)}{\sqrt{n\delta^2\epsilon^4\zeta^6}}} \leq  e^{-x^2},\\
    \sup_{\theta\in \Theta_L} \PP_{\theta}\brackets[\Big]{\norm{\hat{f}_0-f_0} \vee \norm{\hat{f}_1-f_1} >\frac{Cx}{\sqrt{n\delta^2\epsilon^4\zeta^4}}} \leq e^{-x^2}.
  \end{gather*}
  Furthermore, there exist constants $c=c(K)>0$, $\epsilon_1 > 0$ and $\zeta_1 >0$  such that for $\delta\leq 1/6$, $\epsilon\leq \epsilon_1$, $\zeta \leq \zeta_1$, $L\leq 1/3$ and $n\delta^2\epsilon^4\zeta^6\geq 1$,
  \begin{gather*}
    \inf_{\check{\theta}}\sup_{\theta \in \Theta_L} \PP_\theta\brackets[\Big]{\abs{\check{p}-p} \vee \abs{\check{q}-q} >\frac{c(\delta \vee \epsilon \zeta)}{\sqrt{n\delta^2\epsilon^4\zeta^6}}}\geq 1/4,\\
    \inf_{\check{\theta}}\sup_{\theta\in \Theta_L} \PP_{\theta}\brackets[\Big]{\norm{\check{f}_0-f_0} \vee \norm{\check{f}_1-f_1}>\frac{c}{\sqrt{n\delta^2\epsilon^4\zeta^4}}} \geq 1/4,
  \end{gather*}
  where the infima are over all estimators $\check{\theta}=(\check{p},\check{q},\check{f}_0,\check{f}_1).$
\end{theorem}
%	The estimator $\hat{\theta}$ is built via plug-in from those constructed later in \cref{thm:minimax-upper-bound}.
	%Note that the maxima are genuinely required in the lower bounds: in the extreme case where $p$ is close to zero and $q$ is close to 1, one has many samples with $X_i=0$ and few with $X_i=1$, so that $p$ and $f_0$ are easier to estimate accurately than $q$ and $f_1$. Not true: given hardest-to-distinguish pairs for the lower bound here, these will make either f_0 or f_1 hard, and label swapping will make the other hard.
%
The particular value $1/4$ on the right sides is not essential: what is important is that the probabilities are bounded away from zero.

	We deduce immediately the sample complexity for learning the parameters. %Note that w
	We do not seek sharp dependence on $K$ in the bounds because we believe our results can be extended to the nonparametric setting, which we leave for further work. %\textcolor{red}{We do not seek sharp dependence on $K$ in the bounds: rather than calculating this dependence, we believe the interesting extension would to the nonparametric setting, which we leave for further work.} %Let us highlight that our results demonstrate that estimation of the Markov parameters $p$ and $q$ is easier than estimation of the emission densities $f_0$ and $f_1$ if $\delta\leq \zeta$ <- if $f_0,f_1$ are very close, we can estimate both by $\psi_1$, but have no idea what $p,q$ are from that. If $p<<1$ and $q=O(1)$ then can estimate $f_0$ and $p$ easily; estimating $q$ is harder and estimating $f_1$ is hardest.

\begin{corollary}
  \label{cor:learning-rate}
  Fix a target error magnitude $E>0$ and a probability level $\alpha>0$. For the same estimators as in \cref{cor:rates-for-theta}, there exists a constant $C=C(K,L)$ such that for any $\theta\in \Theta_L$ we have
  \begin{multline*}
    n\geq \frac{\log(1/\alpha)}{\delta^2\epsilon^4\zeta^6}\brackets[\Big]{ \frac{C\delta^2}{E^{2}}\bigvee\frac{C\epsilon^2\zeta^2}{E^{2}}\bigvee 1}\\ \implies \PP_{\theta}(\abs{\hat{p}-p} \vee \abs{\hat{q}-q} > E)\leq \alpha,
  \end{multline*}
  and,
  \begin{multline*}
    n\geq \frac{\log(1/\alpha)}{\delta^2\epsilon^4\zeta^4}\brackets[\Big]{\frac{C}{E^2}\bigvee\frac{1}{\zeta^2}}\\
    \implies%
    \PP_{\theta}(\norm{\hat{f_0}-f_0} \vee \norm{\hat{f}_1-f_1}> E)\leq \alpha.
  \end{multline*}
  Conversely there exists a constant $c=c(K)>0$ such that for all $0 < E \leq c(K)$ and for any estimator $\check{\theta}=(\check{p},\check{q},\check{f}_0,\check{f}_1)$ there exists $\theta\in \Theta_L$ such that
  \begin{equation*}
    n\leq \frac{c^2(\delta^2\vee\epsilon^2\zeta^2)}{E^2\delta^2\epsilon^4\zeta^6}
    \implies%
    \PP_\theta( \abs{\check{p}-p} \vee \abs{\check{q}-q} >E)\geq 1/4,
  \end{equation*}
  and,
  \begin{equation*}
    n\leq \frac{c^2}{E^2\delta^2\epsilon^4\zeta^4} %
    \implies \PP_{\theta}(\norm{\check{f}_0-f_0}\vee\norm{\check{f}_1-f_1}>E)\geq 1/4.
  \end{equation*}
\end{corollary}
	Note that to apply \cref{cor:rates-for-theta} for the lower bounds we would initially also need $n\geq (\delta^2\epsilon^4\zeta^6)^{-1}$ but by monotonicity --- i.e.\ the fact that any measurable function of $(Y_1,\dots,Y_n)$ is also a measurable function of $(Y_1,\dots,Y_N)$ for $N\geq n$ --- the restriction can be removed.\\

%	\section{Proof outline}
%	\label{sec:outline}

	Let us sketch the main ideas behind the proof of \cref{cor:rates-for-theta}. The
	full proof is deferred to \cref{sec:proofs}, along with all other proofs for this article.

	The minimax upper bounds are obtained by producing an estimator that attains
	the bounds. Building on the work of \cite{AGHKT14,MR3439359} we know that $\theta$ is
	identifiable from 
	%$p_{\theta}^{(3)}$
	the law of three consecutive observations, and we propose a reparametrisation of the model  to simplify the analysis. Let us denote by $p_{\theta}^{(3)}$ the density of  three consecutive observations.
	%Indeed, 
	Motivated by a desire to simplify the expression for $p_{\theta}^{(3)}$ (see
	\cref{eqn:ThreeStepLaw,eq:17} in
	\cref{sec:change-param}), we introduce new parameters $\phi,\psi$ and we show in \cref{pro:local-equivalence} that
	$\norm{p_{\theta(\phi,\psi)}^{(3)} - p_{\theta(\tilde{\phi},\tilde{\psi})}^{(3)}}$ is equivalent to
	$\rho(\phi,\psi;\tilde{\phi},\tilde{\psi})$, where $\rho$ is defined in the proposition and can be
	seen as %the ``natural'' 
	an adequate statistical distance of the problem (see also the discussion after \cref{kul:upper}). Then, we leverage that $p_{\theta}^{(3)}$ can be estimated in
	Euclidean distance at the parametric rate $n^{-1/2}$ by the empirical estimator
	$\hat{p}_n^{(3)}$ defined in \cref{sec:upper-bounds},
	\cref{lem:estimation-of-p3}. This suggests that solving for
	$(\hat{\phi},\hat{\psi}) \in \argmin_{\phi,\psi} \norm{p_{\theta(\phi,\psi)}^{(3)} - \hat{p}_n^{(3)}}$ will give a good estimator $(\hat{\phi},\hat{\psi})$ for ($\phi,\psi$). By
	standard calculations and using the equivalence between
	$\|p_{\theta(\phi,\psi)}^{(3)} - p_{\theta(\tilde{\phi},\tilde{\psi})}^{(3)}\|$ and
	$\rho(\phi,\psi;\tilde{\phi},\tilde{\psi})$ derived in \cref{pro:local-equivalence}, we obtain bounds on maximum risk of such $(\hat{\phi},\hat{\psi})$ for estimating
	$(\phi,\psi)$ in \cref{thm:minimax-upper-bound}. Finally, the upper bounds for the original parameters in \cref{cor:rates-for-theta} are obtained by taking
	$\hat{\theta} = \theta(\hat{\phi},\hat{\psi})$.

	Incidentally, we remark that the
	parametrisation $(\phi,\psi)$ turns out to be of special interest: the
	components of $\phi$ determine how close the sequence $\bm{Y}$ is to being
	i.i.d in an interpretable way (see \cref{sec:change-param}), and the
	parameter $\psi$ is related to the stationary distribution of the sequence
	$\bm{Y}$. For this reason, we also establish minimax bounds for the
	estimation of $\phi$ and $\psi$ themselves in \cref{thm:minimax-upper-bound,thm:mlb}. %\cref{sec:minimax-lower-bounds}, \cref{thm:mlb}.

	The minimax lower bounds are obtained by an argument \textit{à la} Le Cam. In
	particular, it is a famous result of Le Cam \cite{cam:1986,tsybakov:2009} that the minimax rate (under
	quadratic loss) of estimating a functional $g : \Theta \to \Reals$ is always greater
	than the maximum value that $\abs{g(\theta) - g(\tilde{\theta})}^2$ can take for $\theta,\tilde{\theta} \in \Theta$  under the constraint that
	$\KL(p_{\theta}^{(n)};p_{\tilde{\theta}}^{(n)}) \leq c$, where
	$\KL(p_{\theta}^{(n)};p_{\tilde{\theta}}^{(n)})$ denotes the \textit{Kullback-Leibler} (KL)
	divergence between the laws of $(Y_1,\dots,Y_n)$ under parameters $\theta$ and
	$\tilde{\theta}$, and $0 < c< 1$ is a small positive constant (see \cref{lem:2-point-testing-bound} for the precise formulation we use).
	Understanding bounds on $\abs{g(\theta)-g(\tilde{\theta})}$ in terms of bounds on $\KL(p_{\theta}^{(n)};p_{\tilde{\theta}}^{(n)})$ is also sufficient for obtaining an upper bound on the minimax estimation rate. Since we have dependent observations, the main difficulty of the proof is to relate $\KL(p_{\theta}^{(n)};p_{\tilde{\theta}}^{(n)})$ to
	a suitable notion of distance between $\theta$ and $\tilde{\theta}$. A key result is \cref{kul:upper} showing that under mild assumptions
	$\KL(p_{\theta(\phi,\psi)}^{(n)};p_{\theta(\tilde{\phi},\tilde{\psi})}^{(n)})$ is upper bounded by a constant times $n \rho^2 (\phi,\psi;\tilde{\phi},\tilde{\psi})$.  Then the lower bounds for
	$\phi$ (respectively $\psi$) in \cref{thm:mlb} are obtained by lower bounding
	the value of the optimisation problems $\max\abs{\phi_j - \tilde{\phi}_j}^2$ (respectively
	$\max\abs{\psi_j - \tilde{\psi}_j}^2$) subject to $n \rho^2 (\phi,\psi;\tilde{\phi},\tilde{\psi}) \leq c$ and
	$\theta(\phi,\psi),\theta(\tilde{\phi},\tilde{\psi})\in \Theta$ for a small enough constant $c > 0$. Finally,
	the lower bounds for the original parameters in the
	\cref{cor:rates-for-theta} are essentially deduced from the bounds for
	($\phi,\psi$) and inversion of the parametrisation.

	\section{Change of parametrisation}
	\label{sec:change-param}

	We reparametrise the model in such a way that the i.i.d.\ limiting cases are highlighted, by changing variables to $\phi=(\phi_1,\phi_2,\phi_3)$ and $\psi=(\psi_1,\psi_2)$ defined as
	\begin{gather*}
		%\label{eq:psi}
		\phi(\theta)%
		=%
		\brackets[\big]{\begin{matrix}
				\frac{q-p}{p+q}%
				&1 - p - q%
				&\norm{ f_{0} - f_{1}}
				\\
		\end{matrix}},
		\\
		\psi(\theta)%
		=%
		\brackets[\big]{\begin{matrix}
				\frac{q f_{0} + p f_{1}}{p + q}%
				&\frac{f_0-f_1}{\norm{ f_0-f_1 }}\\
		\end{matrix}}
		.
	\end{gather*}
	Here we have separated the scalar parameters $\phi$ from the vector parameters $\psi$. %It is sometimes more natural to combine $\phi_3$ and $\psi_2$ into a single parameter $\phi_3\psi_2=f_0-f_1$.
	Defining
	\begin{equation}
		\label{eqn:rphi}
		r(\phi)=\tfrac{1}{4} (1-\phi_1^2)\phi_2\phi_3^2,
	\end{equation}
	it follows from the discussion in \cref{sec:setting} that the data $\bm{Y}$ is close to i.i.d.\ exactly when $r(\phi)\approx 0$. [This is of course true also of other combinations of the components of $\phi$, but as \cref{eq:17} will show, $r(\phi)$ is the appropriate combination measuring the ``distance'' to the i.i.d.\ case.]

	Define
	\begin{gather*}
		\Phi=\Phi(\delta,\epsilon,\zeta)=\braces{(\phi(\theta),\psi(\theta)) : \theta\in \Theta(\delta,\epsilon,\zeta)},\\
		\Phi_{L}=\Phi_{L}(\delta,\epsilon,\zeta)
		=	\braces{(\phi(\theta),\psi(\theta)) : \theta \in \Theta_{L}(\delta,\epsilon,\zeta)},
	\end{gather*}
	and note that for $(\phi,\psi)\in \Phi$ we have
	\begin{multline}
          \label{eqn:phi-bounds}
          -\frac{1-\delta}{1+\delta} \leq \phi_1\leq \frac{1-\delta}{1+\delta},\quad %
          \epsilon\leq \abs{\phi_{2}}\leq 1-2\delta,\\ %
          \zeta\leq \phi_3\leq \sqrt{2},\quad %
          \abs{r(\phi)}\geq \delta\epsilon\zeta^2/4,
        \end{multline}
	while for $(\phi,\psi)\in \Phi_{L}$ we additionally have
	\begin{equation}\label{eqn:condition-on-phi2} \abs{\phi_2}\leq 1-L.
	\end{equation}
	\begin{remark}\label{rem:K=2}
		When $K=2$, in view of identifiability issues discussed in the next subsection, $\psi_2$ is not needed in the parametrisation, since we may universally make the choice \[\psi_2=\brackets[\Big]{\frac{1}{\sqrt{2}},-\frac{1}{\sqrt{2}}}.\]
	\end{remark}
	\begin{remark}\label{rem:invert-param}
		The parametrisation $\theta\mapsto (\phi,\psi)$ is invertible: we calculate
		\begin{align*}
			p&= \tfrac{1}{2} (1-\phi_2)(1-\phi_1), \\
			q &= \tfrac{1}{2}(1-\phi_2)(1+\phi_1),\\
			%		f_0 &=\psi_1+\tfrac{1}{2}(1-\phi_1)\phi_3\psi_2,\\
			f_0 &=\psi_1-\tfrac{1}{2}\phi_1\phi_3\psi_2+\tfrac{1}{2}\phi_3\psi_2,\\
			f_1 &=\psi_1-\tfrac{1}{2}\phi_1\phi_3\psi_2-\tfrac{1}{2}\phi_3\psi_2.
			%		f_1&=\psi_1-\tfrac{1}{2}(1+\phi_1)\phi_3\psi_2.
		\end{align*}
	\end{remark}
	\begin{remark}
          \label{rem:valid-phis}
          Suppose $\psi_1$ is a probability density function with respect to counting measure on $\braces{1,\dots,K}$, $\psi_2$ is a function satisfying $\norm{\psi_2}=1$ and $\sum_k \psi_2(k)=0$, and $\phi$ satisfies $\abs{\phi_1}\leq 1$, $\abs{\phi_2}\leq 1$ and $\phi_3\geq 0$. Then $(\phi,\psi)$ lies in $\Phi(\delta,\epsilon,\zeta)$ if and only if
          \begin{equation}
            \label{eqn:condition-on-phi}
            \begin{gathered}
              \tfrac{1}{2}(1-\phi_2)(1-\abs{\phi_1})\geq \delta,\quad%
            \tfrac{1}{2}(1-\phi_2)(1+\abs{\phi_1})\leq 1,\\
            \abs{\phi_2}\geq \epsilon,\quad%
            \phi_3\geq \zeta,
            \end{gathered}
          \end{equation}
          and
          \begin{equation}
            \label{eqn:condition-on-psi}
		\psi_1(k)-\tfrac{1}{2}\phi_1\phi_3\psi_2(k)-\tfrac{1}{2}\phi_3 \abs{\psi_2(k)}\geq 0, \quad \forall k\leq K.
	\end{equation}
 \end{remark}

%	\subsection{Identifiability of the model}
%	\label{sec:identifiability}
	The model \eqref{eqn:model} is identifiable for the parameter set $\Theta$ only up to `label-switching', since $\bm{Y}$ has the same distribution under the parameters $(p,q,f_0,f_1)$ and $(q,p,f_1,f_0)$; in the parametrisation $(\phi,\psi)$, the distribution of $\bm{Y}$ is the same under $(\phi_1,\phi_2,\phi_3,\psi_1,\psi_2)$ and under $(-\phi_1,\phi_2,\phi_3,\psi_1,-\psi_2)$. However, it was proved in \cite{AGHKT14} that aside from this label-switching, the model parameters can be identified from the law of just three consecutive observations. To that end, for any integer $m$ denoting by $P_{\theta}^{(m)}$ the law of $(Y_1,\dots,Y_m)$ under parameter % $m$ consecutive observations with parameter
	$\theta\in\Theta$, and by $p_\theta^{(m)}$ the corresponding density with respect to counting measure on $\braces{1,\dots,K}^m$, we calculate
	\begin{equation}
		\label{eqn:ThreeStepLaw}
		p_{\theta}^{(3)}=\left(\frac{q}{p + q}\right) g \otimes f_0 \otimes g
		+ \left(\frac{p}{p + q}\right) h \otimes f_1 \otimes h,
	\end{equation}
	where $g=(1-p) f_0 + p f_1$ and $h= q f_0 + (1-q) f_1$, and where
	$\otimes$ denotes the tensor product so that
	\[(f\otimes g \otimes h) (a,b,c)= f(a)g(b)h(c),\quad (a,b,c)\in\braces{1,\dots,K}^3.\]
	In the $(\phi,\psi)$ parametrisation, writing just
	$p^{(3)}_{\phi,\psi}$ for $p^{(3)}_{\theta(\phi,\psi)}$ in a slight abuse of notation, we have
	\begin{multline}
		\label{eq:17}
		%\begin{split}
		p^{(3)}_{\phi,\psi} =\psi_1\otimes \psi_1 \otimes \psi_1 + r(\phi)\brackets[\big]{\psi_2\otimes \psi_2 \otimes \psi_1 +\psi_1 \otimes \psi_2 \otimes \psi_2}\\% \\
		+  \phi_2 r(\phi) \psi_2\otimes \psi_1 \otimes \psi_2 - \phi_1\phi_2\phi_3 r(\phi) \psi_2\otimes \psi_2\otimes \psi_2,
		%\end{split}
	\end{multline}
	where we recall the notation $r(\phi)=\tfrac{1}{4}(1-\phi_1^2)\phi_2\phi_3^2$.

	 We define a statistical distance $\rho$ directly on the parameter space $\Phi$ which is equivalent to  the Euclidean distance between the densities $p^{(3)}_{\phi,\psi}$ and $p^{(3)}_{\tilde{\phi},\tilde{\psi}}$. The function $\rho$ is not a true metric because it may not satisfy the triangle inequality and because, due to the identifiability issues reflected by the appearance of factors of $\sign(\ip{\psi_2,\tilde{\psi}_2})$ in its definition, we may have $\rho(\phi,\psi;\tilde{\phi},\tilde{\psi})=0$ with $(\phi,\psi)\neq(\tilde{\phi},\tilde{\psi})$. Here $\ip{\cdot,\cdot}$ denotes the Euclidean inner product on $\RR^K$, \( \ip{f,g}=\sum_{i=1}^K f(k)g(k).\)
	\begin{proposition}
		\label{pro:local-equivalence}
		For $r$ as in \cref{eqn:rphi} define $m$ by
		\begin{equation}\label{eqn:m}
			m(\phi)=(r(\phi),\phi_2 r(\phi), \phi_1 \phi_2 \phi_3 r(\phi) ),
		\end{equation}
		and define $\rho(\phi,\psi;\tilde{\phi},\tilde{\psi})=$
%		\begin{equation}
%                  \label{eq:131}
%                  \begin{aligned}
%                    \max\{%
%                    &\abs{m_1(\phi) - m_1(\tilde{\phi})},\\%
%                    &\abs{m_2(\phi) - m_2(\tilde{\phi})},\,\\
%                    &\abs{m_3(\phi) - \sgn(\ip{\psi_2,\tilde{\psi}_2})\cdot m_3(\tilde{\phi})},\\
%      &(\abs{m_1(\phi)}\vee\abs{m_1(\tilde{\phi})})\cdot\norm{\psi_2 - \sgn(\Inner{\psi_2,\tilde{\psi}_2})\cdot \tilde{\psi}_2},\\
%                    &\norm{\psi_1 - \tilde{\psi}_1} \}.%
%                  \end{aligned}%
%		\end{equation}
		\begin{equation}
	\label{eq:131}
	\begin{aligned}
		\max\{%
		&\abs{m_1(\phi) - m_1(\tilde{\phi})},\abs{m_2(\phi) - m_2(\tilde{\phi})},\,\\
		&\abs{m_3(\phi) - \sgn(\ip{\psi_2,\tilde{\psi}_2})\cdot m_3(\tilde{\phi})}, \norm{\psi_1 - \tilde{\psi}_1}, \\
		&(\abs{m_1(\phi)}\vee\abs{m_1(\tilde{\phi})})\cdot\norm{\psi_2 - \sgn(\Inner{\psi_2,\tilde{\psi}_2})\cdot \tilde{\psi}_2} \}.%
	\end{aligned}%
\end{equation}
		%Note the identifiability constraint that $\psi$ is in a specified half plane does not in general imply $\sgn(\ip{\psi_2,\tilde{\psi}_2})=1$, eg consider $(1,0,\epsilon), (-1,0,\epsilon)$, which both lie in $z>0$ but the inner product is nearly $-1$.
		There exist constants $c_1,c_2>0$ (which depend on $K$) such that for all $(\phi,\psi),(\tilde{\phi},\tilde{\psi})\in \bigcup_{\delta,\epsilon,\zeta}\Phi(\delta,\epsilon,\zeta)$ we have
		\begin{equation*}
			%\label{eq:89}
			c_1 \rho(\phi,\psi;\tilde{\phi},\tilde{\psi})\leq \norm{p_{\phi,\psi}^{(3)} - p_{\tilde{\phi},\tilde{\psi}}^{(3)}}%
			\leq c_2  \rho(\phi,\psi;\tilde{\phi},\tilde{\psi}).
		\end{equation*}
	\end{proposition}

	Optimal estimation rates can be obtained if we adequately understand %well enough
	the Kullback--Leibler divergence between distributions with different parameters. The Kullback--Leibler divergence between $P^{(n)}_{\theta(\phi,\psi)}$ and $P^{(n)}_{\theta(\tilde{\phi},\tilde{\psi})}$ can be related to the statistical distance $\rho(\phi,\psi;\tilde{\phi},\tilde{\psi})$ in a neighbourhood of the independent subcase.

	\begin{proposition}
	\label{prop:kul}
	Assume there exists $c \in (0,1)$ such that $\min(f_0,f_1,\tilde{f}_0,\tilde{f}_1) \geq c$. There exist constants $C,\epsilon_0 > 0$ depending only on $c$ such that if $\max(|\phi_2|,|\tilde{\phi}_2|) \leq \epsilon_0$, then with $\rho$ as in \cref{eq:131},
	\begin{equation*}
	    \KL(P_{\theta(\phi,\psi)}^{(n)},P_{\theta(\tilde{\phi},\tilde{\psi})}^{(n)})%
	    \leq C n \rho(\phi,\psi;\tilde{\phi},\tilde{\psi})^2.
	\end{equation*}
	\label{kul:upper}
	\end{proposition}

    We note that only the lower bound on $\norm{p_{\phi,\psi}^{(3)} - p_{\tilde{\phi},\tilde{\psi}}^{(3)}}$ in \cref{pro:local-equivalence} is used in our paper (it is used in proving \cref{thm:minimax-upper-bound}). The upper bound on $\norm{p_{\phi,\psi}^{(3)} - p_{\tilde{\phi},\tilde{\psi}}^{(3)}}$ is still of interest as it establishes the tightness (up to constants) of the corresponding lower bound, thereby proving the equivalence between $\norm{p_{\phi,\psi}^{(3)} - p_{\tilde{\phi},\tilde{\psi}}^{(3)}}$ and $\rho(\phi,\psi;\tilde{\phi},\tilde{\psi})$ and showing that $\rho$ is %a natural and 
    an   adequate statistical metric for this problem. Furthermore, in combination with \cref{prop:kul}, Pinsker's inequality, and the fact that all norms on the set $\braces{1,\dots,K}^3$ are equivalent, it shows that whenever $\max(|\phi_2|,|\tilde{\phi}_2|)$ is small enough,
    \begin{align*}
      \KL(P_{\theta(\phi,\psi)}^{(n)},P_{\theta(\tilde{\phi},\tilde{\psi})}^{(n)})%
      &\leq C' n \norm{p_{\phi,\psi}^{(3)} - p_{\tilde{\phi},\tilde{\psi}}^{(3)}}^2\\
      &\leq C'' n \KL(P_{\theta(\phi,\psi)}^{(3)},P_{\theta(\tilde{\phi},\tilde{\psi})}^{(3)}),
    \end{align*}
    for constants $C',C'' > 0$, once again highlighting the prominent role of the law of 3 consecutive observations in HMM modelling, and illustrating that optimal estimators (up to numerical constants) can be built solely on the basis of the empirical distribution of blocks of 3 consecutive observations. %Using more than 3 consecutive observations may perhaps help improve in terms of the constants, but won't allow for faster rates.
%    Furthermore, since all norms are equivalent on the set $\braces{1,\dots,K}^3$, we may apply Pinsker's inequality to deduce that
%    %$p_\theta^{(3)}$ is supported on a set of cardinality %Furthermore, since $p_{\theta}^{(3)}$ is supported on a set of cardinality no more than $K^3$, the $2$-norm and the $1$-norm of $p_{\phi,\psi}^{(3)} - p_{\tilde{\phi},\tilde{\psi}}^{(3)}$ are equivalent, which in addition with Pinsker's bound on the $1$-norm establishes that
%    \begin{align*}
%    \KL(P_{\theta(\phi,\psi)}^{(n)},P_{\theta(\phi,\psi)}^{(n)}) \leq C'' n \KL(P_{\theta(\phi,\psi)}^{(3)},P_{\theta(\phi,\psi)}^{(3)}),
%    \end{align*}
%    for some $C''>0$ and whenever $\max(|\phi_2|,|\tilde{\phi}_2|)$ is small enough.
 This shows that as long as the chain $\bm{Y}$ is not ``too dependent'', it behaves almost as if we had observed i.i.d.\ blocks of 3 consecutive observations (in which case we would have that $\KL(P_{\theta(\phi,\psi)}^{(n)},P_{\theta(\tilde{\phi},\tilde{\psi})}^{(n)}) =  (n/3)\KL(P_{\theta(\phi,\psi)}^{(3)},P_{\theta(\tilde{\phi},\tilde{\psi})}^{(3)})$, for all $n$ divisible by $3$). %Notice that for not small $\phi_2$'s, $p_{\theta}^{(3)}$ is enough to recover all parameters of the HMM at the best rate $\sqrt{n}$.

	\section{Upper bounds}\label{sec:upper-bounds}
	We obtain the following upper bounds for estimating $\phi$ and $\psi$. Since we are studying limits as the quantities of interest become small, the relative risk may be of as much interest as the absolute risk, and we provide bounds for both quantities. The bounds demonstrate that learning model parameters is possible in the regime where $n$ is large enough in relation to $\delta$, $\epsilon$ and $\zeta$. Observe firstly that estimation of $p^{(3)}$ is possible at a parametric rate.

	\begin{lemma}\label{lem:estimation-of-p3}
		Define the empirical estimator \\  $\hat{p}^{(3)}_n : \braces{1,\dots,K}^n\to [0,1]$ by \begin{equation}\label{eqn:def:hatp}
		\hat{p}^{(3)}_n(a,b,c)= \frac{1}{n} \sum_{i=1}^{n-2} \II\braces{Y_i=a,~Y_{i+1}=b,~Y_{i+2}=c}.\end{equation} Then for some constant $C=C(K,L)$ and any $x\geq 1$
		\[ \sup_{(\phi,\psi)\in \Phi_L(\delta,\epsilon,\zeta)} \PP_{(\phi,\psi)}(\norm{\hat{p}^{(3)}-p^{(3)}}\geq Cx/\sqrt{n})\leq e^{-x^2}.\]
	\end{lemma}

	%\begin{remark}
	%	The density $p^{(3)}_{\phi,\psi}$ can also be estimated at a typical nonparametric rate when $f_0,f_1$ are modelled nonparametrically [note one obtains the nonparametric rate for density estimation on $\RR$, rather than on $\RR^3$, despite $p^{(3)}$ being a density on $\RR^3$, because of redundant information that allows us to essentially consider only the middle coordinate -- as in Luc Lehericy's paper and in the HMM FDR paper.%, and as seen in the proof for \cref{thm:modulus:max} to follow.
	%\end{remark}
	\begin{theorem}\label{thm:minimax-upper-bound}
		Assume $\Phi_L$ is non-empty and let $\hat{\phi},\hat{\psi}$ be any measurable functions satisfying, for $\hat{p}^{(3)}_n$ as in \cref{eqn:def:hatp}, \begin{equation*}\label{eqn:minimum-distance} \norm{ p^{(3)}_{\hat{\phi},\hat{\psi}}-\hat{p}_n^{(3)}}\leq 2\inf_{(\tilde{\phi},\tilde{\psi})\in\Phi_L} \norm{p^{(3)}_{\tilde{\phi},\tilde{\psi}}-\hat{p}^{(3)}_n}.
		\end{equation*}
		There exists a constant $C=C(K,L)>0$ such that the following hold.
		\begin{enumerate}
                  \item\label{item:mub:phi1} Assume $1\leq x^2\leq n\delta^2 \epsilon^4\zeta^6$. Then
                  \begin{align*}
                    &\sup_{(\phi,\psi)\in\Phi_L(\delta,\epsilon, \zeta)}\PP_{\phi,\psi}\brackets[\Big]{\abs[\Big]{\frac{1-\hat{\phi}_1^2}{1-\phi_1^2}-1}
                    \geq%
                      \frac{\sqrt{2}Cx }{\sqrt{n}\delta\epsilon^2\zeta^3}}\\ %\\
                    &\quad\leq \sup_{(\phi,\psi)\in\Phi_L(\delta,\epsilon, \zeta)}\PP_{\phi,\psi}\brackets[\Big]{\abs{\hat{\phi}_1 - \phi_1} \wedge\abs{\hat{\phi}_1+\phi_1} \geq  \frac{Cx }{\sqrt{n}\epsilon^2\zeta^3}}\\%
                    &\quad%
                      \leq e^{-x^2}.
                  \end{align*}

                  \item\label{item:mub:phi2} Assume $1\leq x^2\leq n\delta^2\epsilon^2\zeta^4$. Then
                  \begin{align*}
                    &\sup_{(\phi,\psi)\in\Phi_L(\delta,\epsilon, \zeta)}\PP_{\phi,\psi}\brackets[\Big]{\abs[\Big]{\frac{\hat{\phi}_2}{\phi_2} - 1}	\geq  \frac{Cx }{\sqrt{n}\delta \epsilon^2 \zeta^2}}\\
                    &\qquad% \\
                    \leq \sup_{(\phi,\psi)\in\Phi_L(\delta,\epsilon, \zeta)}\PP_{\phi,\psi}\brackets[\Big]{\abs{\hat{\phi}_2 - \phi_2}	\geq  \frac{Cx }{\sqrt{n}\delta\epsilon \zeta^2}}\\%
                    &\qquad%
				\leq e^{-x^2}.
                  \end{align*}

                  \item\label{item:mub:phi3} Assume $1\leq x^2\leq n\delta^2\epsilon^4\zeta^6$. Then
                  \begin{align*}
                    &\sup_{(\phi,\psi)\in \Phi_L(\delta,\epsilon,\zeta)} \PP_{\phi,\psi} \brackets[\Big]{ \abs[\Big]{\frac{\hat{\phi}_3}{\phi_3}-1}\geq \frac{Cx}{\sqrt{n}\delta \epsilon^2\zeta^3}}\\% \\
                    &\qquad\leq \sup_{(\phi,\psi)\in \Phi_L(\delta,\epsilon,\zeta)} \PP_{\phi,\psi} \brackets[\Big]{ \abs{\hat{\phi}_3-\phi_3} \geq \frac{Cx }{\sqrt{n}\delta \epsilon^2\zeta^2}}\\
                    &\qquad\leq e^{-x^2}.
                  \end{align*}

		\item\label{item:mub:psi1} Assume $1\leq x^2\leq n$. Then
			\[ \sup_{(\phi,\psi)\in\Phi_L(\delta,\epsilon,\zeta)} \PP_{\phi,\psi}\brackets[\Big]{\norm{\hat{\psi}_1-\psi_1} \geq \frac{Cx }{\sqrt{n}}}\leq e^{-x^2}.\]

		\item\label{item:mub:psi2}
			Assume $1\leq x^2\leq n\delta^2\epsilon^2\zeta^4$ and $K>2$. Then
                        \begin{multline*}
                          \sup_{(\phi,\psi)\in\Phi_L(\delta,\epsilon,\zeta)} \PP_{\phi,\psi}\brackets[\Big]{\norm{\hat{\psi}_2-\psi_2} \wedge \norm{\hat{\psi}_2+\psi_2}\\\geq \frac{Cx }{\sqrt{n}\delta\epsilon\zeta^2}}
                          \leq e^{-x^2}.
                        \end{multline*}
		\end{enumerate}
	\end{theorem}

Recall that estimating $\psi_2$ is unnecessary when $K=2$ (see \cref{rem:K=2}). Note that the absolute loss in each case is bounded, and one can deduce that the bounds for $\phi_2$ and for $\psi$ hold without an upper bound on $x$, with $e^{-x^2}$ on the right replaced by zero (for $C$ large enough).

The estimator proposed in Theorem \ref{thm:minimax-upper-bound} is an approximate solution to an optimisation problem which (by taking squares) is a multivariate polynomial function with $2K$ unknowns. There are several methods in the literature about finding the global optimum of multivariate polynomials, see e.g. \cite{Glopti}. As mentioned in Section \ref{sec:conclusion}, the issue of finding computationally efficient estimation 
methods requires further investigation.

	\section{Lower bounds}
	\label{sec:lower-bounds}
	We prove lower bounds, matching the previous upper bounds in a suitable regime and demonstrating the impossibility of learning model parameters when $n$ is not large enough in relation to $\delta$, $\epsilon$ and $\zeta$.
	%The particular value $1/4$ on the right sides in the following is not essential: what is important is that the probabilities are bounded away from zero. The lower bounds over $\Phi_L$ remain true over the larger set $\Phi$.
	\label{sec:minimax-lower-bounds}
	\begin{theorem}
		\label{thm:mlb}
	%	Grant the compatibility condition \eqref{eqn:compatibility}. 
		There exist constants $c=c(K) > 0$, and $\epsilon_0,\zeta_0 > 0$ such that whenever $\epsilon\leq \epsilon_0$, $\zeta \leq \zeta_0$, $\delta\leq 1/6$ and $L\leq 1/3$ the following hold. [The infima are over all estimators, i.e.\ all measurable functions of the data $(Y_1,\dots,Y_n)$.]
		\begin{enumerate}
			\item\label{item:mlb:phi1}
			Assume $n\delta^2\epsilon^4\zeta^6\geq 1$. Then
			\begin{align*}
			&\inf_{\hat{\phi}_1}  \sup_{(\phi,\psi)\in\Phi_L(\delta,\epsilon, \zeta)}\PP_{\phi,\psi}\brackets[\Big]{\abs{\hat{\phi}_1 - \phi_1} \wedge\abs{\hat{\phi}_1+\phi_1}
					\geq  \frac{c}{\sqrt{n}\epsilon^2 \zeta^3}}\\
                         &\geq  \inf_{\hat{\phi}_1}  \sup_{(\phi,\psi)\in\Phi_L(\delta,\epsilon, \zeta)} \PP_{\phi,\psi}\brackets[\Big]{\abs[\Big]{\frac{1-\tilde{\phi}_1^2}{1-\phi_1^2}-1}\geq \frac{\sqrt{2}c}{\sqrt{n}\delta\epsilon^2\zeta^3}}\\
                          &\geq 1/4.
			\end{align*}

                        \item\label{item:mlb:phi2} Assume $n\delta^2\epsilon^4\zeta^4\geq 1$. Then
                        \begin{align*}
                          &\inf_{\hat{\phi}_2}  \sup_{(\phi,\psi)\in\Phi_L(\delta,\epsilon, \zeta)}\PP_{\phi,\psi}\brackets[\Big]{\abs{\hat{\phi}_2 - \phi_2}
                            \geq   \frac{c}{ \sqrt{n}\delta \epsilon \zeta^2}}\\
                          &\qquad\geq \inf_{\hat{\phi}_2} \sup_{(\phi,\psi)\in\Phi_L(\delta,\epsilon, \zeta)}\PP_{\phi,\psi}\brackets[\Big]{\abs[\Big]{\frac{\hat{\phi}_2}{\phi_2} - 1}
                            \geq \frac{c}{ \sqrt{n}\delta \epsilon^2 \zeta^2}}\\
                          &\qquad\geq 1/4.
                        \end{align*}

			\item\label{item:mlb:phi3} Assume $n\delta^2\epsilon^4\zeta^6\geq 1$. Then
                        \begin{align*}
                          &\inf_{\hat{\phi}_3}  \sup_{(\phi,\psi)\in\Phi_L(\delta,\epsilon, \zeta)}\PP_{\phi,\psi}\brackets[\Big]{\abs{\hat{\phi}_3 - \phi_3}
					\geq  \frac{c}{\sqrt{n}\delta\epsilon^2\zeta^2}} \\
				&\qquad\geq
				\inf_{\hat{\phi}_3}  \sup_{(\phi,\psi)\in\Phi_L(\delta,\epsilon, \zeta)}\PP_{\phi,\psi}\brackets[\Big]{\abs[\Big]{\frac{\hat{\phi}_3}{\phi_3} - 1}
                                  \geq  \frac{c}{\sqrt{n}\delta\epsilon^2\zeta^3}}\\
                          &\qquad\geq 1/4.
                        \end{align*}

			\item\label{item:mlb:psi1} For any $n,$ $\delta$, $\epsilon$ and $\zeta$,
			\[ \inf_{\hat{\psi}_1} \sup_{(\phi,\psi)\in \Phi_L(\delta,\epsilon,\zeta)} \PP_{\phi,\psi}\brackets[\Big]{\norm{\hat{\psi}_1-\psi_1} \geq \frac{c}{\sqrt{n}}}\geq 1/4.\]

                        \item\label{item:mlb:psi2} Assume $n\delta^2\epsilon^2\zeta^4\geq 1$ and  $K>2$. Then
                        \begin{multline*}
                          \inf_{\hat{\psi}_2} \sup_{(\phi,\psi)\in \Phi_L(\delta,\epsilon,\zeta)} \PP_{\phi,\psi}\brackets[\Big]{\norm{\hat{\psi}_2-\psi_2} \wedge \norm{\hat{\psi}_2+\psi_2}\\
                            \geq \frac{c}{\sqrt{n}\delta\epsilon\zeta^2}}\geq 1/4.
                        \end{multline*}
		\end{enumerate}
	\end{theorem}

	\section{Conclusions and future directions}
	\label{sec:conclusion}
	In this work we have quantified the impact on learnability of approaching the i.i.d.\ boundary within the set of parameters of a hidden Markov model. The limiting cases occur when one hidden state is absorbing, when the underlying Markov chain becomes a sequence of independent variables, or when the emission distributions are equal. We have proved both upper and lower bounds for the estimation rates of the parameters in a hidden Markov model with two hidden states and finitely many possible outcomes. Our results characterize the frontier in the parameter space between learnable and unlearnable parameters and quantify how large the sample has to be in order to get estimators with prescribed error with high probability.

	Some tricky regions of the parameter space are not fully captured in the upper and lower bounds. Specifically, the condition on $n$ in the lower bound for estimating $\phi_2$ differs by a factor of $\epsilon^2$ from the corresponding condition in the upper bound.
%	In the upper bound for $\phi_1$, in the region 	$n\delta^2\epsilon^4\zeta^6<x^2\leq n\epsilon^4\zeta^6$ we do not obtain the correct rate of decay in the probability as $x$ increases.
Also, in the upper bound for $\phi_1$, we do not describe the precise estimation behaviour in the region $n\delta^2\epsilon^4\zeta^6<x^2< n\epsilon^4\zeta^6$: in this range we can obtain something by applying the bound with $y^2=\min(x^2,n\delta^2\epsilon^4\zeta^6)$ but we cannot expect that this gives the correct dependence on $x$. [There is no issue in the region $x^2\geq n\epsilon^4\zeta^6$ since we may replace the bound $e^{-x^2}$ with zero, similarly to the comment after the theorem regarding $\phi_2$ and $\psi$.] A similar gap exists for estimating $\phi_3$.  The reason for those gaps is that the inversion formulas given in \cref{pro:modulus:pointwise} are only local; finding global inversion formulas, which would allow the remaining regions to be covered, remains an open problem. 
Our results already work for a wide range of parameters, and extending to the few remaining cases is an interesting issue for future research.

	Regarding the upper bounds, 
	%We also believe our results will  hold in the near-periodic case when the spectral gap is close to 2 and the absolute spectral gap is close to zero.
	our proof method relies on the fact that the two steps of estimating $p^{(3)}$ and of estimating, given $p^{(3)}$, the HMM parameters themselves, decouple. This is because, with good mixing properties for the Markov chain, estimation of $p^{(3)}$ can be done uniformly at a rate not depending on the HMM parameters (Lemma 1). When the spectral gap is small the underlying Markov chain mixes slowly, spending long periods remaining in whichever of the two states it is in, so that estimation of $p^{(3)}$ becomes hard for parameters for which there is small spectral gap. These are not the same parameters for which recovering the HMM parameters given $p^{(3)}$ is most difficult, and so to obtain accurate rates without a spectral gap requires carefully addressing the two steps simultaneously, which is beyond the scope of the paper (we could obtain a suboptimal rate using the current methods just with careful tracking of the spectral gap, since it is lower bounded by $1-2\delta$, but upper and lower bounds obtained in this way mismatch by a factor of $\delta$). Note the above arguments explain the requirement for a spectral gap, not an absolute spectral gap; we believe our results will in fact hold in the near-periodic case when the spectral gap is close to 2 and the absolute spectral gap is close to zero, but this would require some extra technical calculations in the proof of \cref{lem:estimation-of-p3}.

We believe similar results hold with more than two hidden states and with arbitrary nonparametric emission distributions. Investigation of the fundamental limits for learning more general HMMs and misspecified modelling will be the object of further work. Developments of our findings for clustering, multiple testing and sparse settings will also be the object of further work, and all will depend fundamentally on the results obtained here.

We analysed a minimum distance estimator for theoretical convenience, and it is possible that the same upper bounds hold for more practical estimators (for example empirical least squares estimators and tensor-based methods).	On the practical side, usual estimation algorithms can be expected to exhibit bad computational behaviour when the unknown true parameters lie near the learning frontier, as shown in \cite{Phys:2020} for  the EM algorithm.
	We have not tackled this issue here and we believe it merits substantive investigation, both in building robust algorithms and in detecting the poor performance in the problematic region. This last question is interesting both from a practical and a theoretical point of view.

\section{Proofs}
\label{sec:proofs}

\subsection{Proof of \cref{pro:local-equivalence}}
Recall the definition \eqref{eqn:m} of $m$ as
\begin{equation*}
  m(\phi)=(r(\phi),\phi_2 r(\phi), \phi_1 \phi_2 \phi_3 r(\phi) ),\ %
  r(\phi)=\tfrac{1}{4}(1-\phi_1^2)\phi_2\phi_3^2.
\end{equation*}
We write $\tilde{m}=m(\tilde{\phi})$, and we write $\psi_{ijk}$ for $\psi_i\otimes \psi_j\otimes \psi_k$ and $\tilde{\psi}_{ijk}$ for $\tilde{\psi}_i\otimes \tilde{\psi}_j \otimes \tilde{\psi}_k$. Then from \cref{eq:17} we have
\begin{equation}
  \label{eq:24}
  \begin{aligned}
    p_{\phi,\psi}^{(3)} - p_{\tilde{\phi},\tilde{\psi}}^{(3)}%
    &= (\psi_{111} - \tilde{\psi}_{111})\\%
    &\quad%
      + \braces{m_1(\psi_{221} + \psi_{122}) - \tilde{m}_1(\tilde{\psi}_{221} +
			\tilde{\psi}_{122})}\\%
    &\quad%
      + \braces{m_2\psi_{212} -
      \tilde{m}_2\tilde{\psi}_{212}}\\%
    &\quad%
      - \braces{m_3\psi_{222} - \tilde{m}_3\tilde{\psi}_{222} }.
  \end{aligned}
\end{equation}
Recalling that $\Inner{\cdot,\cdot}$ denotes the Euclidean inner product on $\RR^K$, we have $\ip{\psi_1,1} = 1$, $\Inner{\psi_2,1} = 0$,
$\norm{\psi_2}=1$ and $\norm{1}= K^{1/2}$. Let $\Inner{\cdot,\cdot}$ also denote the Euclidean inner product on $\RR^{K \times K \times K},$ wherein for functions $f_i,\tilde{f}_i:\braces{1,\dots,K}\to \RR,~i\leq 3$ we have
\[ \ip{f_1\otimes f_2\otimes f_3,\tilde{f}_1\otimes \tilde{f}_2\otimes \tilde{f}_3} = \ip{f_1,\tilde{f}_1}\ip{f_2,\tilde{f}_2}\ip{f_3,\tilde{f}_3}.\]
\paragraph{Lower bounding $\norm{p^{(3)}_{\phi,\psi}-p^{(3)}_{\tilde{\phi},\tilde{\psi}}}$}

For any function \\ $f:\braces{1,\dots,K}\to \RR$, we have
\begin{equation*}
	%\label{eq:72}
	\Inner{p_{\phi,\psi}^{(3)} - p_{\tilde{\phi},\tilde{\psi}}^{(3)}, f \otimes 1\otimes 1}%
	= \Inner{\psi_{111} - \tilde{\psi}_{111},f\otimes 1 \otimes 1}
	= \Inner{\psi_1 - \tilde{\psi}_1, f}.
\end{equation*}
Then
\begin{align}
	\norm{\psi_1 - \tilde{\psi}_1}%
	&= \sup_{\norm{f}=1}\abs{\ip{\psi_1 - \tilde{\psi}_1, f}} \nonumber \\%
	&= \sup_{\norm{f}=1}\abs{\ip{p_{\phi,\psi}^{(3)} - p_{\tilde{\phi},\tilde{\psi}}^{(3)}, f \otimes 1\otimes 1}} \nonumber \\
	&\leq \norm{p_{\phi,\psi}^{(3)} - p_{\tilde{\phi},\tilde{\psi}}^{(3)}} \sup_{\norm{f}=1}\norm{f\otimes 1 \otimes 1} \nonumber\\
  	\label{eq:73}
  &= K\norm{p_{\phi,\psi}^{(3)} - p_{\tilde{\phi},\tilde{\psi}}^{(3)}},
\end{align}
and similarly,
\begin{align}
  &\Inner{p_{\phi,\psi}^{(3)} - p_{\tilde{\phi},\tilde{\psi}}^{(3)},1 \otimes f \otimes f}\nonumber\\%
	&= \Inner{\psi_{111} - \tilde{\psi}_{111}, 1 \otimes f \otimes f} + \Inner{m_1\psi_{122} - \tilde{m}_1\tilde{\psi}_{122},1\otimes f \otimes f} \nonumber \\
	\label{eq:74}
	&= \Inner{\psi_1 - \tilde{\psi}_1,f}^2 + m_1 \Inner{\psi_2,f}^2 - \tilde{m}_1\Inner{\tilde{\psi}_2,f}^2.
\end{align}
Choosing
$f = \psi_2 + \sgn(\Inner{\psi_2,\tilde{\psi}_2})\cdot \tilde{\psi}_2$ (with the convention that $\sign(0)=+1$), we observe that
\begin{equation*}
	%	\label{eq:45}
	\Inner{\psi_2,f} = 1 + |\Inner{\psi_2,\tilde{\psi}_2}|%
	= \sgn(\Inner{\psi_2,\tilde{\psi}_2})\cdot \Inner{\tilde{\psi}_2,f}.
\end{equation*}
In particular we note that
$\Inner{\psi_2,f}^2 = \Inner{\tilde{\psi}_2,f}^2 = (1 + |\Inner{\psi_2,\tilde{\psi}_2}|)^2 \geq 1$.
Since also $\norm{f}^2=2+2\abs{\ip{\psi_2,\tilde{\psi}_2}}\leq 4$, returning to \eqref{eq:74} we observe that
\begin{align}
	\abs{m_1 - \tilde{m}_1}%
	&\leq \norm{f}^2\norm{\psi_1-\tilde{\psi}_1}^2 + \norm{1\otimes f\otimes f}\norm{p^{(3)}_{\phi,\psi}-p^{(3)}_{\tilde{\phi},\tilde{\psi}}} \nonumber \\
	&\leq 4\norm{\psi_1-\tilde{\psi}_1}^2+4K^{1/2}\norm{p^{(3)}_{\phi,\psi}-p^{(3)}_{\tilde{\phi},\tilde{\psi}}} \nonumber \\
	\label{eq:76}
	&\leq 4(K^{7/2}+K^{1/2}) \norm{p_{\phi,\psi}^{(3)} - p_{\tilde{\phi},\tilde{\psi}}^{(3)}},
\end{align}
where for the last line we have used \cref{eq:73} and the fact that $\norm{p_{\phi,\psi}^{(3)} - p_{\tilde{\phi},\tilde{\psi}}^{(3)}}^2 \leq K^3$. We
continue by considering the expression $f \otimes 1 \otimes f$, for which we have
\begin{align*}
	%\label{eq:77}
	&\Inner{p_{\phi,\psi}^{(3)} - p_{\tilde{\phi},\tilde{\psi}}^{(3)},f \otimes 1 \otimes f}\\%
	&= \Inner{\psi_{111} - \tilde{\psi}_{111}, f \otimes 1 \otimes f}%
	+ \Inner{m_2 \psi_{212} - \tilde{m}_2\tilde{\psi}_{212}, f \otimes 1 \otimes f}\\
	&= \Inner{\psi_1 - \tilde{\psi}_1,f}^2 + m_2\Inner{\psi_2,f}^2 - \tilde{m}_2 \Inner{\tilde{\psi}_2,f}^2
\end{align*}
Recognising symmetry with \cref{eq:74}, we again choose $f = \psi_2 + \sgn(\Inner{\psi_2,\tilde{\psi}_2})\cdot\tilde{\psi}_2$ to obtain
\begin{equation}
	\label{eq:78}
	\abs{m_2 - \tilde{m}_2}%
	\leq 4(K^{7/2}+K^{1/2}) \norm{p_{\phi,\psi}^{(3)} - p_{\tilde{\phi},\tilde{\psi}}^{(3)}}.
\end{equation}
Finally, considering the expression $f \otimes f \otimes f$, we observe that $\Inner{p_{\phi,\psi}^{(3)} - p_{\tilde{\phi},\tilde{\psi}}^{(3)},f \otimes f \otimes f}$ is equal to
\begin{equation*}
	%	\label{eq:79}
	\begin{split}
&\Inner{\psi_{111} - \tilde{\psi}_{111}, f \otimes f \otimes f}%
                  + \Inner{m_1(\psi_{221}+\psi_{122})\\
          &\qquad- \tilde{m}_1(\tilde{\psi}_{221}+\tilde{\psi}_{122}),f\otimes f \otimes f}\\%
		&\qquad+ \Inner{m_2 \psi_{212} - \tilde{m}_2\tilde{\psi}_{212}, f \otimes f \otimes f}\\
          &\qquad- \Inner{m_3 \psi_{222} - \tilde{m}_3 \tilde{\psi}_{222},f \otimes f \otimes f}.
	\end{split}
\end{equation*}
In other words, $\Inner{p_{\phi,\psi}^{(3)} - p_{\tilde{\phi},\tilde{\psi}}^{(3)},f \otimes f \otimes f}$ equals
\begin{equation*}
	%	\label{eq:46}
  \begin{split}
    &\Inner{\psi_1 - \tilde{\psi}_1,f}^3%\\%
      %    &\quad%
		+ 2\brackets[\big]{m_1\Inner{\psi_2,f}^2\Inner{\psi_1,f} - \tilde{m}_1\Inner{\tilde{\psi}_2,f}^2\Inner{\tilde{\psi}_1,f}  }\\
		&\quad+ \brackets[\big]{m_2\Inner{\psi_2,f}^2\Inner{\psi_1,f} - \tilde{m}_2\Inner{\tilde{\psi}_2,f}^2\Inner{\tilde{\psi}_1,f}  }\\
          &\quad- \brackets[\big]{ m_3 \Inner{\psi_2,f}^3 - \tilde{m}_3 \Inner{\tilde{\psi}_2,f}^3  }.
	\end{split}
\end{equation*}
Once more choosing $f = \psi_2 + \sgn(\Inner{\psi_2,\tilde{\psi}_2})\cdot\tilde{\psi}_2$, we
obtain (recall that by construction $\norm{f} \leq 2$,
$1\leq \Inner{\psi_2,f}^2 = \Inner{\tilde{\psi}_2,f}^2 \leq 4$, and also
$\sgn(\Inner{\psi_2,\tilde{\psi}_2})\Inner{\psi_2,f} = \Inner{\tilde{\psi}_2,f}$)
\begin{align*}
	%	\label{eq:80}
	&\abs{m_3 - \sgn(\Inner{\psi_2,\tilde{\psi}_2})\cdot \tilde{m}_3}\\%
	&\leq 8\norm{\psi_1 - \tilde{\psi}_1}^3+ 8\norm{p_{\phi,\psi}^{(3)} - p_{\phi,\psi}^{(3)}}\\
	&\quad+ 8\abs[\big]{ m_1\Inner{\psi_1,f} - \tilde{m}_1\Inner{\tilde{\psi}_1,f}}
	+ 4\abs[\big]{m_2\Inner{\psi_1,f} - \tilde{m}_2\Inner{\tilde{\psi}_1,f}}.
\end{align*}
For some constant $C=C(K)$ we have
\begin{align*}
	%	\label{eq:81}
	&\abs[\big]{ m_1\Inner{\psi_1,f} - \tilde{m}_1 \Inner{\tilde{\psi}_1,f}}\\%
	&\qquad\leq \abs{\Inner{\psi_1,f}}\abs[\big]{ m_1 - \tilde{m}_1}%
	+ \abs{\tilde{m}_1}\abs{\Inner{\psi_1 - \tilde{\psi}_1,f}}\\
	&\qquad\leq 2\norm{\psi_1}\abs[\big]{ m_1 - \tilde{m}_1} %
	+ 2 \abs{\tilde{m}_1}\norm{\psi_1 - \tilde{\psi}_1}\\
	&\qquad\leq C \norm{p_{\phi,\psi}^{(3)} - p_{\tilde{\phi},\tilde{\psi}}^{(3)}},
\end{align*}
where for the last line we have used \cref{eq:73,eq:76} and that $\norm{\psi_1}\leq K^{1/2}$ and $\abs{\tilde{m}_1}\leq \tilde{\phi}_3^2/4 \leq \norm{\tilde{f}_0-\tilde{f}_1}^2/4\leq K/4$.

Similarly, using \cref{eq:78} and the fact that $\abs{\tilde{m}_2}$ is suitably bounded, we have for some $C=C(K)$
\begin{align*}
	%\label{eq:83}
	\abs[\big]{ m_2\Inner{\psi_1,f} - \tilde{m}_2\Inner{\tilde{\psi}_1,f}}%
	&\leq C \norm{p_{\phi,\psi}^{(3)} - p_{\tilde{\phi},\tilde{\psi}}^{(3)}}.
\end{align*}
We deduce for some different constant $C=C(K)$ that
\begin{equation}
	\label{eq:84}
	\abs{m_3 - \sgn(\Inner{\psi_2,\tilde{\psi}_2})\cdot \tilde{m}_3}%
	\leq C \norm{p_{\phi,\psi}^{(3)} - p_{\tilde{\phi},\tilde{\psi}}^{(3)}}.
\end{equation}
Finally, for $\psi_2$ we show that for some $C$ we have
\begin{equation}
	\label{eq:49}
	(\abs{m_1}\vee\abs{\tilde{m}_1})\norm{\psi_2 - \sgn(\ip{\psi_2,\tilde{\psi}_2})\cdot\tilde{\psi}_2}%
	\leq C \norm{p_{\phi,\psi}^{(3)} - p_{\tilde{\phi},\tilde{\psi}}^{(3)}}.
\end{equation}
If $\psi_2=\tilde{\psi}_2$ there is nothing to prove, so we assume without loss of generality that $\psi_2\neq \tilde{\psi}_2$. Also assume that $\abs{m_1}\geq \abs{\tilde{m}_1}$, the final bound then following by symmetry.
Returning to \cref{eq:74} with $f$ to be chosen, we see that
\begin{align*}
	%	\label{eq:85}
	m_1\brackets[\big]{ \Inner{\psi_2,f}^2 - \Inner{\tilde{\psi}_2,f}^2}%
	&= \Inner{\psi_1 - \tilde{\psi}_1,f}^2\\%
	&\quad%
          - \Inner{p_{\phi,\psi}^{(3)} - p_{\tilde{\phi},\tilde{\psi}}^{(3)},1\otimes f\otimes f}\\%
  &\quad
	- \Inner{\tilde{\psi}_2,f}^2 \brackets{ m_1 - \tilde{m}_1}.
\end{align*}
Since
$\Inner{\psi_2,f}^2 -\Inner{\tilde{\psi}_2,f}^2 = \Inner{\psi_2 - \tilde{\psi}_2,f}\Inner{\psi_2 + \tilde{\psi}_2,f}$
we obtain
\begin{multline}
	\label{eq:86}
	\abs{m_1\Inner{\psi_2 - \tilde{\psi}_2,f}\Inner{\psi_2 + \tilde{\psi}_2,f}}%
	\leq \norm{f}^2\norm{\psi_1 - \tilde{\psi}_1}^2\\%
	+ K \norm{f}^2\norm{p_{\phi,\psi}^{(3)} - p_{\tilde{\phi},\tilde{\psi}}^{(3)}}%
	+ \abs{m_1 - \tilde{m}_1}\Inner{\tilde{\psi}_2,f}^2.
\end{multline}
Observe that $\psi_2+\tilde{\psi}_2$ is orthogonal to $\psi_2 - \tilde{\psi}_2$ (this arises from the fact that $\psi_2$ and $\tilde{\psi}_2$ have unit norms) and choose
\begin{equation*}
	%	\label{eq:87}
	f%
	= \frac{\psi_2 + \tilde{\psi}_2}{\norm{\psi_2 + \tilde{\psi}_2}} +\frac{ \psi_2 - \tilde{\psi}_2}{\norm{\psi_2 - \tilde{\psi}_2}};
\end{equation*}
note that
\begin{equation*}
	%	\label{eq:88}
	\Inner{\psi_2 - \tilde{\psi}_2,f}\Inner{\psi_2 + \tilde{\psi}_2,f}%
	= \norm{\psi_2 - \tilde{\psi}_2} \norm{\psi_2 + \tilde{\psi}_2}.
\end{equation*}
Since also $\norm{f}\leq 2$ and $\abs{\Inner{\tilde{\psi}_2,f}} \leq 2$, continuing from \cref{eq:86} and using \cref{eq:73,eq:76} we see that for a constant $C=C(K)$
\begin{align*}
	%	\label{eq:94}
	&\abs{m_1}\norm{\psi_2 - \tilde{\psi}_2} \norm{\psi_2 + \tilde{\psi}_2}\\%
	&\qquad\leq 4\norm{\psi_1 - \tilde{\psi}_1}^2 + 4K\norm{p_{\phi,\psi}^{(3)} - p_{\tilde{\phi},\tilde{\psi}}^{(3)}}%
	+  4\abs{m_1 - \tilde{m}_1}\\
	&\qquad\leq 2C  \norm{p_{\phi,\psi}^{(3)} - p_{\tilde{\phi},\tilde{\psi}}^{(3)}}.%
\end{align*}
Observing that
\begin{align*}
	%	\label{eq:48}
	&\norm{\psi_2 - \tilde{\psi}_2}^2\norm{\psi_2 + \tilde{\psi}_2}^2\\%
	&\qquad=\norm{\psi_2 - \sgn(\Inner{\psi_2,\tilde{\psi}_2})\cdot\tilde{\psi}_2}^2 \norm{\psi_2 + \sgn(\Inner{\psi_2,\tilde{\psi}_2})\cdot \tilde{\psi}_2}^2\\
	&\qquad=\norm{\psi_2 - \sgn(\Inner{\psi_2,\tilde{\psi}_2})\cdot\tilde{\psi}_2}^2\brackets[\big]{2 + 2 \abs{\Inner{\psi_2,\tilde{\psi}_2}} }\\
	&\qquad\geq 2\norm{\psi_2 - \sgn(\Inner{\psi_2,\tilde{\psi}_2})\cdot\tilde{\psi}_2}^2,
\end{align*}
and recalling we assumed that $\abs{m_1}\geq \abs{\tilde{m}_1}$, \cref{eq:49} follows.

The proof that $\norm{p^{(3)}_{\phi,\psi}-p^{(3)}_{\tilde{\phi},\tilde{\psi}}}$ is lower bounded up to a constant by $\rho(\phi,\psi;\tilde{\phi},\tilde{\psi})$ follows by combining \cref{eq:73,eq:76,eq:78,eq:84,eq:49}

\paragraph{Upper bounding $\norm{p^{(3)}_{\phi,\psi}-p^{(3)}_{\tilde{\phi},\tilde{\psi}}}$}
From \cref{eq:24},
\begin{equation}
	\label{eq:54}
	\begin{split}
		\norm{p_{\phi,\psi}^{(3)} - p_{\tilde{\phi},\tilde{\psi}}^{(3)}}%
		&\leq \norm{\psi_{111} - \tilde{\psi}_{111}}%
                  +\abs{m_1 - \tilde{m}_1}\norm{\psi_{221} + \psi_{122}}\\%
          &\quad
            + \abs{\tilde{m}_1}\norm{\psi_{221} - \tilde{\psi}_{221}}%
            + \abs{\tilde{m}_1}\norm{\psi_{122} - \tilde{\psi}_{122}}\\
          &\quad
		+\abs{m_2 - \tilde{m}_2} \norm{\psi_{212}}%
		+ \abs{\tilde{m}_2}\norm{\psi_{212} - \tilde{\psi}_{212}}\\
		&\quad +\abs{m_3 - \tilde{m}_3} \norm{\psi_{222}}%
		+ \abs{\tilde{m}_3}\norm{\psi_{222} - \tilde{\psi}_{222}}.
	\end{split}
\end{equation}
Note that the bound remains valid if we replace the final two terms by
\[ \abs{m_3+\tilde{m}_3} \norm{\psi_{222}} +\abs{\tilde{m}_3}\norm{\psi_{222}+\tilde{\psi}_{222}};\] we focus on the case where $\sign(\ip{\psi_2,\tilde{\psi}_2})=+1$ for which the original decomposition yields suitable bounds, but the proof in the other case is similar using the alternative decomposition.

As used already in proving the lower bound on $\norm{p^{(3)}_{\phi,\psi}-p^{(3)}_{\tilde{\phi},\tilde{\psi}}}$, we note that
\begin{equation*} %\label{eqn:coeffs-bounded}
	\max(\norm{\psi_{221}},\norm{\psi_{122}},\abs{\tilde{m}_1},\norm{\psi_{212}},\abs{\tilde{m}_2},\norm{\psi_{222}},\abs{\tilde{m}_3})\leq C,
\end{equation*}
for some $C=C(K)$.
To conclude the proof it thus suffices to bound the tensor product terms $\norm{\psi_{ijk}-\tilde{\psi}_{ijk}}$ in terms of the differences $\norm{\psi_1-\tilde{\psi}_1}, \norm{\psi_2-\tilde{\psi}_2}$.
First we decompose
\begin{multline*}
	%\label{eq:43}
	\norm{\psi_1 \otimes \psi_1 \otimes \psi_1 - \tilde{\psi}_1 \otimes \tilde{\psi}_1 \otimes \tilde{\psi}_1}%
	\leq  \norm{\psi_1 \otimes \psi_1 \otimes \psi_1 - \tilde{\psi}_1 \otimes \psi_1 \otimes \psi_1}\\%
	\quad	+ \norm{\tilde{\psi}_1 \otimes \psi_1 \otimes \psi_1 - \tilde{\psi}_1 \otimes \tilde{\psi}_1 \otimes \psi_1}%
	+ \norm{\tilde{\psi}_1 \otimes \tilde{\psi}_1 \otimes \psi_1 - \tilde{\psi}_1 \otimes \tilde{\psi}_1 \otimes \tilde{\psi}_1},
\end{multline*}
so that
\begin{align}
  \norm{\psi_{111}-\tilde{\psi}_{111}}%
  &\leq \norm{\psi_1-\tilde{\psi}_1}(\norm{\psi_1}^2+\norm{\psi_1}\norm{\tilde{\psi}_1}+\norm{\tilde{\psi}_1}^2) \nonumber\\
  \label{eq:44}
  &\leq 3K\norm{\psi_1-\tilde{\psi}_1}.
\end{align}
We also note, recalling that $\psi_2$ and $\tilde{\psi}_2$ have unit norms, that
\begin{align*}
	%	\label{eq:97}
	\norm{\psi_{221} - \tilde{\psi}_{221}}^2%
	&= \norm{\psi_{221}}^2 + \norm{\tilde{\psi}_{221}}^2 - 2\Inner{\psi_{221},\tilde{\psi}_{221}}\\
	&= \norm{\psi_1}^2 + \norm{\tilde{\psi}_1}^2 - 2\Inner{\psi_2,\tilde{\psi}_2}^2\Inner{\psi_1,\tilde{\psi}_1}\\
	&= \norm{\psi_1 - \tilde{\psi}_1}^2 +2 \Inner{\psi_1,\tilde{\psi}_1}\brackets[\big]{1-\Inner{\psi_2,\tilde{\psi}_2}^2}\\
	&\leq \norm{\psi_1 - \tilde{\psi}_1}^2  + 2 \norm{\psi_1}\norm{\tilde{\psi}_1} \abs{1-\Inner{\psi_2,\tilde{\psi}_2}^2}.
\end{align*}
Observe that
\begin{equation}
	\label{eq:98}
	\norm{\psi_2 - \tilde{\psi}_2}^2%
	= 2(1 - \Inner{\psi_2,\tilde{\psi}_2}),
\end{equation}
and hence
\begin{align*}
	%	\label{eq:99}
  \abs[\big]{1-\ip{\psi_2,\tilde{\psi}_2}^2}
  &= \abs[\big]{1+\ip{\psi_2,\tilde{\psi_2}}}\abs[\big]{1-\ip{\psi_2,\tilde{\psi}_2}}\\
  &\leq 2\abs{1-\ip{\psi_2,\tilde{\psi}_2}}\\
  &=\norm{\psi_2-\tilde{\psi}_2}^2.
\end{align*}
We deduce that
\begin{align}
	\norm{\psi_{221} - \tilde{\psi}_{221}}^2%
  &\leq \norm{\psi_1 - \tilde{\psi}_1}^2 + 2 \norm{\psi_1}\norm{\tilde{\psi}_1} \norm{\psi_2 - \tilde{\psi}_2}^2 \nonumber\\
  \label{eq:100}
  &\leq \norm{\psi_1 - \tilde{\psi}_1}^2+2K\norm{\psi_2-\tilde{\psi}_2}^2.
\end{align}
By symmetry, the same bound holds for $\norm{\psi_{122} - \tilde{\psi}_{122}}$ and for $\norm{\psi_{212}-\tilde{\psi}_{212}}$. Furthermore $\norm{\psi_{222}} = 1$, and using \eqref{eq:98},
\begin{align}
	\label{eq:101}
	\norm{\psi_{222} - \tilde{\psi}_{222}}^2%
	&= \norm{\psi_{222}}^2 + \norm{\tilde{\psi}_{222}}^2 - 2 \Inner{\psi_{222},\tilde{\psi}_{222}} \nonumber \\
	&= 2 - 2\Inner{\psi_2,\tilde{\psi}_2}^3 \nonumber \\
	&= 2\brackets[\big]{1-\ip{\psi_2,\tilde{\psi}_2}}\brackets[\big]{1+\ip{\psi_2,\tilde{\psi}_2}+\ip{\psi_2,\tilde{\psi}_2}^2}\nonumber \\
	&\leq 3\norm{\psi_2 - \tilde{\psi}_2}^2.
\end{align}
The claim follows from inserting \cref{eq:44,eq:100,eq:101} into \cref{eq:54}.

\subsection{Proof of \cref{kul:upper}}

%We write $P_{\theta}$ the joint distribution of $(\bm{X},\bm{Y})$ under $\theta$. As in the document
Write $X_{1:k}$ and $Y_{1:k}$ for the vectors $(X_1,\dots, X_k)$ and $(Y_1,\dots,Y_k)$ respectively, and recall that $P_\theta^{(n)}$ denotes the law of $Y_{1:n}$ for parameter $\theta$. Without loss of generality we may assume that $\sgn(\langle \psi_2,\tilde{\psi}_2 \rangle) = +1$, %to facilitate the analysis, so that $\rho(\phi,\psi;\tilde{\phi},\tilde{\psi})$ gets nicer. Indeed,
since one may substitute  $\tilde{\phi}' = (-\tilde{\phi}_1,\tilde{\phi}_2,\tilde{\phi}_3)$ and $\tilde{\psi}' = (\tilde{\psi}_1,-\tilde{\psi}_2)$ for $\tilde{\phi}$ and $\tilde{\psi}$ and obtain $P_{\tilde{\theta}}^{(n)} = P_{\tilde{\theta'}}^{(n)}$, hence $\KL(P_{\theta}^{(n)}; P_{\tilde{\theta}}^{(n)}) = \KL(P_{\theta}^{(n)}; P_{\tilde{\theta}'}^{(n)})$, but $\sgn(\ip{ \psi_2, \tilde{\psi}_2'}) = - \sgn(\ip{\psi_2,\tilde{\psi}_2})$.
% if $\sgn(\langle \psi_2,\tilde{\psi}_2) = -1$ we may substitute $\tilde{\theta} = (\tilde{\phi},\tilde{\psi})$ for $\tilde{\theta}' = (\tilde{\phi}',\tilde{\psi}')$ with $\tilde{\phi}' = (-\tilde{\phi}_1,\tilde{\phi}_2,\tilde{\phi}_3)$ and $\tilde{\psi}' = (\tilde{\psi}_1,-\tilde{\psi}_2)$ and have that $P_{\tilde{\theta}}^{(n)} = P_{\tilde{\theta'}}^{(n)}$, hence $\KL(P_{\theta}^{(n)}; P_{\tilde{\theta}}^{(n)}) = \KL(P_{\theta}^{(n)}; P_{\tilde{\theta}'}^{(n)})$ but $\sgn(\langle \psi_2, \tilde{\psi}_2' \rangle) = +1$.
%Recall that $P_{\theta}^{(n)}$ denotes the law of
%  the first $n$ observations $(Y_1,\dots,Y_{n})$ under $\theta$. For
%  $k = 1,\dots,n$ we use the alias $Y_{1:k}$ for the vector $(Y_1,\dots,Y_{k})$,
 % similarly $X_{1:k}$ for $(X_1,\dots,X_{k})$. The proof for $n = 1$ is rather immediate, we now assume $n\geq2$.
 Recall that
  $\KL(P;Q)$ is upper bounded by the chi-square distance
  $\chi^2(P,Q)=\EE_Q[(\intd P/ \intd Q -1)^2]$ (e.g.\ \cite[Lemma
  2.7]{tsybakov:2009}). Then using that
  $\PP_{\theta}(Y_1 = \cdot) = \psi_1(\cdot)$ and
  $\PP_{\tilde{\theta}}(Y_1 = \cdot) = \tilde{\psi}_1(\cdot) \geq c$, we have
  \begin{align}
	\KL(P_{\theta}^{(1)};P_{\tilde{\theta}}^{(1)})
    &\leq \sum_{y \in \mathcal{Y}} \frac{[\PP_{\theta}(Y_1 = y) - \PP_{\tilde{\theta}}(Y_1 = y)  ]^2}{\PP_{\tilde{\theta}}(Y_1 = y)}\nonumber\\%
    \label{eq:kl:29}
    &\leq \frac{\|\psi_1 - \tilde{\psi}_1\|^2}{c}.
  \end{align}
  This yields the case $n=1$ since the definition \eqref{eq:131} implies that $\rho(\phi,\psi;\tilde{\phi},\tilde{\psi})\geq \norm{\psi_1-\tilde{\psi}_1}^2$.

 Now assume that $n\geq 2$. By the chain rule for relative divergence (used inductively), we have
 \begin{multline}
	\label{eq:kl:11}
	\KL(P_{\theta}^{(n)};P_{\tilde{\theta}}^{(n)})
	= \KL(P_{\theta}^{(1)};P_{\tilde{\theta}}^{(1)})\\%
   + \sum_{k=1}^{n-1} \EE_{\theta}[\KL(\PP_{\theta}(Y_{k+1} \in \cdot \mid Y_{1:k});\PP_{\tilde{\theta}}( Y_{k+1}\in  \cdot \mid Y_{1:k}))].
  \end{multline}
  The first term was addressed above, and we now consider the remaining terms. Again bounding the KL divergence by the chi-square distance, we have
  \begin{align*}
    &\KL(\PP_{\theta}(Y_{k+1} \in \cdot \mid Y_{1:k});\PP_{\tilde{\theta}}( Y_{k+1}\in  \cdot \mid Y_{1:k}))\\%
	&\qquad\leq%
   \sum_{y\in \mathcal{Y}} \frac{[\PP_{\theta}(Y_{k+1} = y \mid Y_{1:k})-  \PP_{\tilde{\theta}}(Y_{k+1} = y \mid Y_{1:k})]^2}{\PP_{\tilde{\theta}}(Y_{k+1} = y \mid Y_{1:k})}\\
	&\qquad\leq \frac{\|\PP_{\theta}(Y_{k+1} = \cdot \mid Y_{1:k})-  \PP_{\tilde{\theta}}(Y_{k+1} = \cdot \mid Y_{1:k}) \|^2}{\min_{y\in \mathcal{Y}} \PP_{\tilde{\theta}}(Y_{k+1} = y \mid Y_{1:k}) }.
  \end{align*}
  But, for any $k \geq 1$, noting that $\PP_{\theta}(Y_{k+1} = y \mid Y_{1:k}, X_{1:k+1} = x) = f_{x_{k+1}}(y)$,
  \begin{align}
    \notag
    &\PP_{\theta}(Y_{k+1} = y\mid Y_{1:k})\\%
	% &=\sum_{x\in \{0,1\}^{k+1}} \PP_{\theta}(Y_{k+1} = y \mid Y_{1:k}, X_{1:k+1} = x)\PP_{\theta}(X_{1:k+1}=x \mid Y_{1:k})\\
	\notag
	&\qquad= \sum_{x\in \{0,1\}^{k+1}} f_{x_{k+1}}(y) \PP_{\theta}(X_{1:k+1}=x \mid Y_{1:k})\\
	\label{eq:kl:3}
	&\qquad= \sum_{x\in \{0,1\}} f_{x}(y)\PP_{\theta}(X_{k+1}=x\mid Y_{1:k}),
  \end{align}
  where we have used that $Y_{k+1} \mid (Y_{1:k},X_{1:k+1})$ has the same law as
  $Y_{k+1}\mid X_{k+1}$. Therefore when $\min(\tilde{f}_0,\tilde{f}_1) \geq c$
  we must have, for all $Y_{1:k}$ and all $k \geq 1$,
  \begin{align*}
	\PP_{\tilde{\theta}}(Y_{k+1} = y \mid Y_{1:k})%
	&\geq c \sum_{x\in \{0,1\}} \PP_{\tilde{\theta}}(X_{k+1}=x\mid Y_{1:k})%
	= c.
  \end{align*}
  Hence we have established that, for all $Y_{1:k}$ and all $k \geq 1$,
  \begin{multline}
	\label{eq:kl:13}
	\KL(\PP_{\theta}(Y_{k+1} \in \cdot \mid Y_{1:k});\PP_{\tilde{\theta}}( Y_{k+1}\in  \cdot \mid Y_{1:k}))\\%
	\leq \frac{\|\PP_{\theta}(Y_{k+1} = \cdot \mid Y_{1:k})-  \PP_{\tilde{\theta}}(Y_{k+1} = \cdot \mid Y_{1:k}) \|^2}{c }.
  \end{multline}
  Let us now rewrite $\PP_{\theta}(Y_{k+1} = y \mid Y_{1:k})$ in the
  parametrisation $(\phi,\psi)$. For convenience we introduce the notation
  $P_k(x) \coloneqq \PP_{\theta}(X_{k+1} = x \mid Y_{1:k} )$ for the prediction
  filters, and we similarly write
  $\tilde{P}_k(x) \coloneqq \PP_{\tilde{\theta}}(X_{k+1} = x \mid Y_{1:k} )$. By
  \cref{eq:kl:3,rem:invert-param},
  \begin{align*}
	&\PP_{\theta}(Y_{k+1} = y \mid Y_{1:k})\\%
	&\qquad= f_0(y) P_k(0) + f_1(y)P_k(1)\\
	&\qquad= \Big(\psi_1(y) - \tfrac{1}{2}\phi_1\phi_3\psi_2(y) + \tfrac{1}{2}\phi_3\psi_2(y)  \Big)P_k(0)\\%
	&\qquad\quad%
   + \Big(\psi_1(y) - \tfrac{1}{2}\phi_1\phi_3\psi_2(y) - \tfrac{1}{2}\phi_3\psi_2(y)  \Big)P_k(1)\\
	&\qquad= \psi_1(y) + \tfrac{1}{2}\big(P_k(0) - P_k(1) - \phi_1 \big)\phi_3\psi_2(y).
  \end{align*}
  For $k\geq 1$, define
  \begin{align*}
	V_k
	&\coloneqq%
   \phi_3(P_k(0) - P_k(1) - \phi_1)% \\
    = \phi_3(1 - 2P_k(1) - \phi_1),\\%
    \tilde{V}_k&\coloneqq \tilde{\phi}_3(1-2\tilde{P}_k(1)-\tilde{\phi}_1).
  \end{align*}
  %Similarly we will let $\tilde{V}_k$ for the same quantity when $\theta$ is
%  replaced by $\tilde{\theta}$. It follows from this definition and the
 % combination of
  Then combining \cref{eq:kl:11,eq:kl:29,eq:kl:13}, we obtain that $\KL(P_{\theta}^{(n)};P_{\tilde{\theta}}^{(n)})$ is bounded above by
  \begin{equation*}
    \frac{\|\psi_1 - \tilde{\psi}_1\|^2}{c}%
   + \frac{1}{c}\sum_{k=1}^{n-1}\EE_{\tilde{\theta}}%
   \Big[\Big\|\psi_1 - \tilde{\psi}_1 + \frac{1}{2}V_k \psi_2 - \frac{1}{2}\tilde{V}_k\tilde{\psi}_2 \Big\|^2 \Big]
 \end{equation*}
 which is in turn bounded above by
 \begin{equation*}
   \frac{2n-1}{c}\|\psi_1 - \tilde{\psi}_1\|^2%
   + \frac{1}{2c}\sum_{k=1}^{n-1}\EE_{\tilde{\theta}}[\|V_k\psi_2 - \tilde{V}_k\tilde{\psi}_2\|^2]
 \end{equation*}
 so that in the end, using that $\|\tilde{\psi}_2\|^2 = 1$,
  \begin{multline}
    \label{eq:kl:39}
    \KL(P_{\theta}^{(n)};P_{\tilde{\theta}}^{(n)})
	\leq \frac{2n-1}{c}\|\psi_1 - \tilde{\psi}_1\|^2\\%
   + \frac{\|\psi_2 - \tilde{\psi}_2\|^2}{c}\sum_{k=1}^{n-1}\EE_{\tilde{\theta}}[V_k^2]%
   + \frac{1}{c}\sum_{k=1}^{n-1}\EE_{\tilde{\theta}}[(V_k - \tilde{V}_k)^2].
  \end{multline}

  Let us now find an inductive formula for $V_k$. Let us define $P_k(x) \coloneqq \PP_{\theta}(X_{k+1} = x \mid Y_{1:k})$. First we observe that for any $k \geq 2$, since $\PP_{\theta}(X_{k+1} = x \mid Y_{1:k},X_k = x') = \PP_{\theta}(X_{k+1} = x \mid X_k = x')$,
  \begin{align*}
    P_k(x)
	&= \sum_{x' \in \{0,1\}} \PP_{\theta}(X_{k+1} = x \mid X_k = x')\PP_{\theta}(X_k = x' \mid Y_{1:k})\\
	&= \sum_{x' \in \{0,1\}} Q_{x',x}\PP_{\theta}(X_k = x' \mid Y_{1:k-1},Y_k),
  \end{align*}
  and we further calculate
  \begin{align*}
	&\PP_{\theta}(X_k = x' \mid Y_{1:k-1},Y_k = y_k)\\%
	&= \frac{\PP_{\theta}(X_k=x',Y_k = y_k \mid Y_{1:k-1})}{\PP_{\theta}(Y_k \mid Y_{1:k-1})}\\
	&= \frac{f_{x'}(y_k)\PP_{\theta}(X_k = x' \mid Y_{1:k-1})}{ \sum_{x'' \in \{0,1\}}\PP_{\theta}(y_k \mid Y_{1:k-1},X_{k} = x'')\PP_{\theta}(X_k = x'' \mid Y_{1:k-1}) }\\
	&= \frac{f_{x'}(y_k)P_{k-1}(x')}{\sum_{x'' \in \{0,1\}}f_{x''}(y_k)P_{k-1}(x'')  }.
  \end{align*}
  Similarly, for $k=1$,
  \begin{align*}
	&\PP_{\theta}(X_2 = x \mid Y_1 = y_1)\\%
	&= \frac{\PP_{\theta}(X_2 = x, Y_1=y_1)}{\PP_{\theta}(y_1)}\\
	&= \frac{\sum_{x' \in \{0,1\}}\PP_{\theta}(X_2 = x, Y_1 = y_1 \mid X_1 = x')\PP_{\theta}(X_1 = x') }{\sum_{x'\in \{0,1\}}f_{x'}(y_1)\PP_{\theta}(X_1=x') }\\
	&= \frac{\sum_{x' \in \{0,1\}} f_{x'}(y_1) Q_{x',x} \PP_{\theta}(X_1 = x') }{\sum_{x'\in \{0,1\}}f_{x'}(y_1)\PP_{\theta}(X_1=x')}.
  \end{align*}
  To summarise, we have proved the recursive formula %That is we have the induction
  \begin{align*}
	P_k(x)%
	&=
   \begin{cases}
	 \frac{\sum_{x' \in \{0,1\}} Q_{x',x} f_{x'}(Y_k)P_{k-1}(x')}{\sum_{x' \in \{0,1\}} f_{x'}(Y_k)P_{k-1}(x')} &\mathrm{if}\ k \geq 2,\\
	 \frac{\sum_{x' \in \{0,1\}} f_{x'}(Y_1) Q_{x',x} \PP_{\theta}(X_1 = x') }{\sum_{x'\in \{0,1\}}f_{x'}(Y_1)\PP_{\theta}(X_1=x')} &\mathrm{if}\ k=1.
   \end{cases}
  \end{align*}
  Therefore when $k \geq 2$, $V_k$ equals
  \begin{equation*}
    \phi_3\Big(1
   - 2\frac{Q_{0,1} f_0(Y_k)P_{k-1}(0) + Q_{1,1}f_1(Y_k)P_{k-1}(1)}{f_0(Y_k)P_{k-1}(0) + f_1(Y_k)P_{k-1}(1)} - \phi_1\Big)
 \end{equation*}
 which rewrites as
  \begin{align*}
	% V_k%
	% &= \\
	% &= \phi_3\Big(1 -  \frac{2p f_0(Y_k)P_{k-1}(0) + 2(1 - q)f_1(Y_k)P_{k-1}(1)}{f_0(Y_k)P_{k-1}(0) + f_1(Y_k)P_{k-1}(1)} - \phi_1\Big)\\
	 \phi_3\Big(1 - \frac{2pf_{0}(Y_k) + 2P_{k-1}(1)[(1-q)f_1(Y_k) - pf_0(Y_k) ]}{f_0(Y_k) + P_{k-1}(1)[f_1(Y_k) - f_0(Y_k)]} - \phi_1\Big).%\\
  \end{align*}
  We write for convenience
  \begin{equation*}
  	D_k%
	= f_0(Y_k) + P_{k-1}(1)[f_1(Y_k) - f_0(Y_k)],
  \end{equation*}
  and,
  \begin{multline*}
    N_k
    = (1 - \phi_1)\phi_3D_k - 2\phi_3pf_{0}(Y_k)\\
    - 2\phi_3P_{k-1}(1)[(1-q)f_1(Y_k) - pf_0(Y_k) ],
  \end{multline*}
  so that $V_k = N_k/D_k$. We rewrite the previous expressions solely in terms of the parameters $(\phi,\psi)$ [recall the inversion formulae in \cref{rem:invert-param}].
  First, $D_k$ is equal to
  \begin{align*}
	&\psi_1(Y_k) - \frac{1}{2}\phi_1\phi_3\psi_2(Y_k) + \frac{1}{2}\phi_3\psi_2(Y_k)%
   - P_{k-1}(1) \phi_3\psi_2(Y_k)\\
	&= \psi_1(Y_k) + \frac{1}{2}\phi_3\psi_2(Y_k)[1 - 2 P_{k-1}(1)  - \phi_1  ]\\
	&= \psi_1(Y_k) + \frac{V_{k-1}}{2} \psi_2(Y_k).
  \end{align*}
  Also,
  \begin{align*}
	2p f_0%
	&= (1 - \phi_1)(1 - \phi_2)\Big[\psi_1 - \frac{1}{2}\phi_1\phi_3\psi_2 + \frac{1}{2}\phi_3\psi_2 \Big],% \\
	% &= (1 - \phi_1 - \phi_2 + \phi_1\phi_2)\Big[\psi_1 - \frac{1}{2}\phi_1\phi_3\psi_2 + \frac{1}{2}\phi_3\psi_2 \Big]
  \end{align*}
  and $(1-q)f_1 - pf_0$
  \begin{align*}
	&= (1 - q - p)f_1 - p(f_0 - f_1)\\
	&= \phi_2\Big(\psi_1 - \frac{1}{2}\phi_1\phi_3\psi_2 - \frac{1}{2}\phi_3\psi_2   \Big)%
   - \frac{1}{2}(1 - \phi_2)(1 - \phi_1)\phi_3\psi_2\\
	&= \phi_2 \psi_1 - \frac{1}{2}\phi_3\psi_2\Big(\phi_2 + \phi_1\phi_2 + (1-\phi_2)(1-\phi_1) \Big)\\
	&= \phi_2\psi_1 - \frac{1}{2}\phi_3\psi_2\Big(1 - \phi_1 + 2\phi_1\phi_2\Big).
  \end{align*}
  Using the last three displays and the fact that
  $2\phi_3P_{k-1}(1) = -V_{k-1} + \phi_3 - \phi_1\phi_3$, we obtain that
  \begin{align*}
	N_k%
	&=(1-\phi_1)\phi_3\Big[\psi_1(Y_k) + \frac{V_{k-1}}{2}\psi_2(Y_k) \Big]\\%
        &\quad-\phi_3(1 - \phi_1)(1 - \phi_2)\\
    &\qquad\times \Big[\psi_1(Y_k) - \frac{1}{2}\phi_1\phi_3\psi_2(Y_k) + \frac{1}{2}\phi_3\psi_2(Y_k) \Big]\\%
        &\quad+(V_{k-1} - \phi_3 + \phi_1\phi_3 )\\
    &\qquad\times \Big[\phi_2\psi_1(Y_k) - \frac{1}{2}\phi_3\psi_2(Y_k)\Big(1 - \phi_1 + 2\phi_1\phi_2\Big) \Big].
  \end{align*}
  Grouping together the terms proportional to $V_{k-1}$ and the others, $N_k$ equals
  \begin{multline*}
    V_{k-1}\Big[\phi_2\psi_1(Y_k)\\
    + \frac{1}{2}\Big(-\phi_3 + \phi_1\phi_3 - 2\phi_1\phi_2\phi_3 + (1-\phi_1)\phi_3 \Big)\psi_2(Y_k) \Big]\\
   + \psi_1(Y_k)\Big[
 (1-\phi_1)\phi_3 - \phi_3(1 - \phi_1)(1 - \phi_2) - \phi_3(1-\phi_3)\phi_2
   \Big]\\
   + \frac{1}{2}\psi_2(Y_k)\Big[-\phi_3^2(1-\phi_1)^2(1-\phi_2)%
   + \phi_3^2(1-\phi_1)(1 - \phi_1 + 2\phi_1\phi_2)
   \Big].
  \end{multline*}
  We remark that
  \begin{align*}
	-\phi_3 + \phi_1\phi_3 - 2\phi_1\phi_2\phi_3 + (1-\phi_1)\phi_3%
	&= -2 \phi_1\phi_2\phi_3,
  \end{align*}
  and
  \begin{align*}
	(1-\phi_1)\phi_3 - \phi_3(1 - \phi_1)(1 - \phi_2) - \phi_3(1-\phi_3)\phi_2%
	&= 0,
  \end{align*}
  and
  \begin{align*}
	&-\phi_3^2(1-\phi_1)^2(1-\phi_2)%
   + \phi_3^2(1-\phi_1)(1 - \phi_1 + 2\phi_1\phi_2)\\
	&= -\phi_3^2(1-\phi_1)^2(1-\phi_2) + \phi_3^2(1 - \phi_1)^2 + 2\phi_1\phi_2\phi_3^2(1-\phi_1)\\
	&= \phi_2\phi_3^2(1 - \phi_1)^2 + 2 \phi_1\phi_2\phi_3^2(1 - \phi_1)\\
	&= \phi_2\phi_3^2[1 - 2\phi_1 + \phi_1^2 + 2\phi_1 - 2\phi_1^2]\\
	&= \phi_2\phi_3^2(1 - \phi_1^2).
  \end{align*}
  That is,
  \begin{multline*}
	N_k%
    = \phi_2V_{k-1} \big[\psi_1(Y_k) - \phi_1\phi_3\psi_2(Y_k)\big]\\%
   + \frac{\phi_2\phi_3^2(1 - \phi_1^2)}{2}\psi_2(Y_k),
  \end{multline*}
  which means that for $k \geq 2$,
  \begin{align*}
	V_k%
	&= \frac{\phi_2[\psi_1(Y_k) - \phi_1\phi_3\psi_2(Y_k)]V_{k-1} + 2r(\phi)\psi_2(Y_k)}{\psi_1(Y_k) + \frac{1}{2}\psi_2(Y_k)V_{k-1}}.
  \end{align*}
  For $k =1$, recalling that $Y_1\sim \psi_1$ and $\PP_\theta(X_1=1)=p/(p+q)$, we have
  \begin{align*}
	V_1%
	&= \phi_3(1 - 2P_1(1) - \phi_1)\\
%	&= \phi_3\Big(1 - \phi_1 - 2 \frac{f_0(Y_1)Q_{0,1} \frac{q}{p+q} + f_1(Y_1)Q_{1,1} \frac{p}{p+q}}{\psi_1(Y_1)} \Big)\\
	&= \phi_3\Big(1 - \phi_1 - 2 \frac{f_0(Y_1)\frac{p q}{p+q} + f_1(Y_1)\frac{(1-q)p}{p+q}}{\psi_1(Y_1)} \Big)\\
%	&= \phi_3\Big(1 - \phi_1 - 2 \frac{ f_1(Y_1)\frac{p}{p+q}   + [f_0(Y_1) - f_1(Y_1)]\frac{p q}{p+q}}{\psi_1(Y_1)} \Big)\\
	&= \phi_3\Big(1 - \phi_1 - 2 \frac{ f_1(Y_1)\frac{p}{p+q}
   + \phi_3 \psi_2(Y_1)\frac{p q}{p+q}}{\psi_1(Y_1)} \Big)\\
	% &= \phi_3\Big(1 - \phi_1 - 2 \frac{[\psi_1(Y_1) - \frac{1}{2}\phi_1\phi_3\psi_2(Y_1) - \frac{1}{2}\phi_3\psi_2(Y_1)]\frac{1 - \phi_1}{2} + \phi_3\psi_2(Y_1)\frac{(1-\phi_2)(1-\phi_1^2)}{4} }{\psi_1(Y_1)} \Big)\\
	&= -\phi_3^2\psi_2(Y_1) \frac{ -\frac{(1-\phi_1)(1+\phi_1)}{2} + \frac{(1-\phi_2)(1-\phi_1^2)}{2} }{\psi_1(Y_1)}\\
	&= \frac{\frac{1}{2}(1-\phi_1^2)\phi_2\phi_3^2\psi_2(Y_1)}{\psi_1(Y_1)},
  \end{align*}
  where to go from the third to fourth line we have used the expressions derived in  \cref{rem:invert-param} for $f_1$, $p$ and $q$. Letting $m_1 = r(\phi)$, $m_2 = r(\phi)\phi_2$, and
  $m_3 = r(\phi)\phi_1\phi_2\phi_3$, we have obtained the inductive formula
  \begin{align*}
	V_k =
	\begin{cases}
	  \frac{[m_2\psi_1(Y_k) - m_3\psi_2(Y_k)] \frac{V_{k-1}}{m_1} + 2m_1 \psi_2(Y_k) }%
	  {\psi_1(Y_k) + \frac{1}{2}\psi_2(Y_k)V_{k-1}} &\mathrm{if}\ k\geq 2,\\
	  \frac{2m_1 \psi_2(Y_1)}{\psi_1(Y_1)} &\mathrm{if}\ k=1.
	\end{cases}
  \end{align*}
  The strategy is now to bound $V_k - \tilde{V}_k$ for $k \geq 2$ in terms of
  $V_1 - \tilde{V}_1$ using the above inductive formula. To do so, we will need
  an upper bound for $V_k$ (respectively $\tilde{V}_k$) which we establish
  now. We claim that $|V_k| \leq 4|m_1|/c$ for all $k\geq 1$ provided $\epsilon_0$ is taken small enough. Indeed,
  $|\psi_2(Y_1)| \leq \|\psi_2\| = 1$ and $c\leq \psi_1(Y_1) \leq 1$, hence
  $|V_1| \leq 2 |m_1|/c \leq 4|m_1|/c$. Now suppose that
  $|V_{k-1}| \leq 4|m_1|/c$; then, under the assumptions of the \lcnamecref{prop:kul} with for $\epsilon_0=\epsilon_0(c)$ small enough, using \cref{eqn:phi-bounds} to see that $|m_1| \leq |\phi_2| \leq \epsilon_0$, $|\phi_1 \phi_2 \phi_3| \leq \sqrt{2}|\phi_2| \leq \sqrt{2}\epsilon_0$, we have
  \begin{align}
	\notag
	|V_k|%
	&\leq \frac{(|m_2| + |m_3|)\frac{4}{c} + 2|m_1|}{c - \frac{1}{2}\frac{4|m_1|}{c} }\\
	\notag
	&\leq |m_1| \frac{(|\phi_2| + |\phi_1\phi_2\phi_3|)\frac{4}{c} + 2 }{c - 4 |m_1|/c}\\
	\label{eq:kl:57}
	&\leq \frac{4|m_1|}{c}.
  \end{align}
  %because, using \cref{eqn:phi-bounds}, $|m_1| \leq |\phi_2| \leq \epsilon_0$, $|\phi_1 \phi_2 \phi_3| \leq \sqrt{2}|\phi_2| \leq \sqrt{2}\epsilon_0$ and $\epsilon_0=\epsilon_0(c)$ is assumed to be small enough.
  Similarly $|\tilde{V}_k| \leq 4 |\tilde{m}_1|/c$ for all $k \geq 1$. We are
  now in position to bound $V_k - \tilde{V}_k$ for $k\geq 2$. Recall
  $V_k = N_k/D_k$ and similarly write $\tilde{V}_k = \tilde{N}_k/\tilde{D}_k$. Then
  \begin{align*}
	V_k - \tilde{V}_k%
	&= \frac{N_k}{D_k} - \frac{\tilde{N}_k}{\tilde{D}_k}\\%
	&= \frac{\tilde{D}_kN_k - D_k \tilde{N}_k}{D_k\tilde{D}_k}\\%
	&= \frac{(\tilde{D}_k - D_k)N_k}{D_k\tilde{D}_k}%
	+ \frac{N_k - \tilde{N}_k}{\tilde{D}_k}.
  \end{align*}
  As when bounding $\abs{V_k}$, we can assume that $\epsilon_0$ is small enough to have $D_k \geq c/2$ and $\tilde{D}_k \geq c/2$, and
  \begin{align*}
	|N_k|%
	&\leq (|m_2| + |m_3|)\frac{4}{c} + 2|m_1|%\\
	\leq 4|m_1|.
  \end{align*}
  Therefore,
  \begin{align}
	\label{eq:kl:59}
	|V_k - \tilde{V}_k|%
	&\leq \frac{16|m_1|}{c^2}|D_k - \tilde{D}_k|%
   + \frac{2}{c}|N_k - \tilde{N}_k|.
  \end{align}
  But, recalling the definition \eqref{eq:131} of $\rho$, we have
  \begin{align*}
    \notag
	&|D_k - \tilde{D}_k|\\%
	&\quad= \Big| \psi_1(Y_k) - \tilde{\psi}_1(Y_k) + \frac{1}{2}\Big(\psi_2(Y_k)V_{k-1} - \tilde{\psi}_2(Y_k)\tilde{V}_{k-1} \Big) \Big|\\
	\notag
	&\quad\leq |\psi_1(Y_k) - \tilde{\psi}_1(Y_k)|%
          + \frac{|\psi_2(Y_k)|}{2}|V_{k-1} - \tilde{V}_{k-1}|\\%
    &\qquad
   + \frac{|\tilde{V}_{k-1}|}{2}|\psi_2(Y_k) - \tilde{\psi}_2(Y_k)|\\
   \notag
	&\quad\leq \|\psi_1 - \tilde{\psi}_1\| + \frac{|V_{k-1} - \tilde{V}_{k-1}|}{2}%
   + \frac{2|m_1|}{c}\|\psi_2 - \tilde{\psi}_2\|\\
	%\label{eq:kl:60}
	&\quad\leq \Big(1 + \frac{2}{c}\Big)\rho(\phi,\psi;\tilde{\phi},\tilde{\psi})%
   + \frac{|V_{k-1} - \tilde{V}_{k-1}|}{2},
  \end{align*}
  and the difference $N_k - \tilde{N}_k $ is equal to
  \begin{align*}
	& [m_2\psi_1(Y_k) - \tilde{m}_2\tilde{\psi}_1(Y_k) - m_3\psi_2(Y_k) + \tilde{m}_3\tilde{\psi}_2(Y_k)] \frac{V_{k-1}}{m_1}\\
	&\quad%
          + \Big(\frac{V_{k-1}}{m_1} - \frac{\tilde{V}_{k-1}}{\tilde{m}_1} \Big)\Big(\tilde{m}_2\tilde{\psi}_1(Y_k) - \tilde{m}_3\tilde{\psi}_2(Y_k)  \Big)\\%
    &\quad
   + 2m_1\psi_2(Y_k) - 2\tilde{m}_1 \tilde{\psi}_2(Y_k),
  \end{align*}
%  from which we deduce that
%  \begin{align*}
%	|N_k - \tilde{N}_k|%
%    &\leq \Big(|m_2|\|\psi_1 - \tilde{\psi}_1\| + |m_2 - \tilde{m}_2|\\
%    &\qquad%
%   + |m_3\|\psi_2 - \tilde{\psi}_2\| + |m_3 - \tilde{m}_3|
%   \Big)\Big|\frac{V_{k-1}}{m_1} \Big|\\
%	&\quad%
%          + \Big(|V_{k-1} - \tilde{V}_{k-1}|\\
%    &\qquad+ |m_1 - \tilde{m}_1| \Big|\frac{V_{k-1}}{m_1} \Big| \Big)\frac{|\tilde{m}_2| + |\tilde{m}_3|}{|\tilde{m}_1|}\\
%	&\quad%
%   + 2|m_1| \|\psi_2 - \tilde{\psi}_2\| + 2|m_1 - \tilde{m}_1|\\
%	&\leq 4|m_1 - \tilde{m}_1|%
%   + \frac{4|m_2 - \tilde{m}_2|}{c}%
%   + \frac{4|m_3 - \tilde{m}_3|}{c}\\%
%	&\quad%
%          + \frac{4|m_2|\|\psi_1 - \tilde{\psi}_1\|}{c}\\%
%    &\quad%
%      + \Big(\frac{4|m_3|}{c|m_1|} + 2\Big)|m_1|\|\psi_2 - \tilde{\psi}_2\|\\
%    &\quad%
%   + \frac{c|V_{k-1} - \tilde{V}_{k-1}|}{8},
%  \end{align*}
%  where the last line holds when $\epsilon_0$ is small enough.

  from which we deduce that $\abs{N_k-\tilde{N}_k}$ is upper bounded by 
\begin{multline*}
\Big(|m_2|\|\psi_1 - \tilde{\psi}_1\| + |m_2 - \tilde{m}_2|
	+ |m_3\|\psi_2 - \tilde{\psi}_2\| + |m_3 - \tilde{m}_3|
	\Big)\Big|\frac{V_{k-1}}{m_1} \Big| \\	+ \Big(|V_{k-1} - \tilde{V}_{k-1}| + |m_1 - \tilde{m}_1| \Big|\frac{V_{k-1}}{m_1} \Big| \Big)\frac{|\tilde{m}_2| + |\tilde{m}_3|}{|\tilde{m}_1|} \\
	+ 2|m_1| \|\psi_2 - \tilde{\psi}_2\| + 2|m_1 - \tilde{m}_1|, \end{multline*} 
which, for $\epsilon_0>0$ small enough, is further upper bounded by
\begin{multline*}
	 4|m_1 - \tilde{m}_1|%
	+ \frac{4|m_2 - \tilde{m}_2|}{c}%
	+ \frac{4|m_3 - \tilde{m}_3|}{c}%
	+ \frac{4|m_2|\|\psi_1 - \tilde{\psi}_1\|}{c}
	\\ + \Big(\frac{4|m_3|}{c|m_1|} + 2\Big)|m_1|\|\psi_2 - \tilde{\psi}_2\| + \frac{c|V_{k-1} - \tilde{V}_{k-1}|}{8}.
\end{multline*}
 Inserting these bounds into \cref{eq:kl:59}, we
  find that there is a constant $B$ depending solely on $c$ such that for all
  $k \geq 2$
  \begin{align*}
	|V_k - \tilde{V}_k|%
	&\leq B \rho(\phi,\psi;\tilde{\phi},\tilde{\psi}) + \Big(\frac{1}{4} + \frac{8|m_1|}{c^2} \Big)|V_{k-1} - \tilde{V}_{k-1}|\\
	&\leq B \rho(\phi,\psi;\tilde{\phi},\tilde{\psi}) + \frac{1}{2}|V_{k-1} - \tilde{V}_{k-1}|,
  \end{align*}
  again when $\epsilon_0$ is small enough. Hence for $k \geq 2$,
  \begin{align*}
  |V_k - \tilde{V}_k|
  &\leq 2(1 - 2^{-k})B \rho(\phi,\psi;\tilde{\phi},\tilde{\psi})%
  + 2^{1-k} |V_1 - \tilde{V}_1|\\
  &\leq \frac{3B}{2}\rho(\phi,\psi;\tilde{\phi},\tilde{\psi})%
  + \frac{|V_1 - \tilde{V}_1|}{2}.
  \end{align*}
  To finish the proof, it is enough to show that $|V_1 - \tilde{V}_1|$ is bounded by a constant multiple of $\rho(\phi,\psi;\tilde{\phi},\tilde{\psi})$, which follows from its definition and the same arguments as above. Thus for some constant $B' > 0$ depending only on $c$
  \begin{equation}
    \label{eq:kl:98}
    \max_{k=1,\dots,n}|V_k - \tilde{V}_k|
    \leq B'\rho(\phi,\psi;\tilde{\phi},\tilde{\psi}).%
  \end{equation}
  The conclusion follows by combining \cref{eq:kl:39,eq:kl:57,eq:kl:98}.

\subsection{Proof of \cref{thm:minimax-upper-bound}}
We start with the proof of \cref{lem:estimation-of-p3}, that $p^{(3)}$ can be estimated at a parametric rate.
\begin{proof}[Proof of \cref{lem:estimation-of-p3}]
	We use a Markov chain concentration result from \cite{Paulin2015}.
	%Can't use the reversible result because we apply to the "three step chain" not the original chain, so we lose reversibility. For estimating $\psi_1$, could use the reversible one, and then we only need a spectral gap ($1-\phi_2>L$) not a pseudo-spectral gap ($1-\abs{\phi_2}>L$)
	Theorem 3.4 therein (but note there is an updated
	version of the paper on arXiv) tells us that for any stationary Markov chain $\bm{Z}=(Z^{(1)},Z^{(2)},\dots)$ of pseudo-spectral gap $\gamma_{\textnormal{ps}}$ (defined as in \cite{Paulin2015}) and any function $h$ satisfying $\EE [h(Z^{(1)})^2] \leq \sigma^2$ and $\norm{h}_\infty\leq b$,
	\begin{multline}\label{eqn:MC-concentration} \PP(\abs{\sum_{i=1}^n h(Z^{(i)})- E h(Z^{(1)})} \geq x)\\
          \leq 2 \exp \brackets[\Big]{-\frac{x^2 \gamma_{\textnormal{ps}}}{8(n+1/\gamma_{\textnormal{ps}})\sigma^2+20bx}}. \end{multline}
	We apply to the chain $\bm{Z}$ %=(Z^{(1)},Z^{(2)},\dots)$
	defined by $Z^{(n)}=(X_n,X_{n+1},X_{n+2},Y_n,Y_{n+1},Y_{n+2})$; we begin by showing the pseudo-spectral gap of this chain is bounded from below. Proposition 3.4 of the same reference shows that the reciprocal of the pseudo-spectral gap of any chain is bounded above by twice the mixing time $t_{\textnormal{mix}}^{\bm{Z}}$ of the chain, defined as the first time that the law of $\bm{Z}$, regardless of the starting distribution, is within 1/4 of its invariant distribution in total variation distance. We note that $t_{\textnormal{mix}}^{\bm{Z}}$ is equal to the mixing time $t_{\textnormal{mix}}^{\bm{X}^{(3)}}$ of the chain $((X_n,X_{n+1},X_{n+2})_{n\geq 0})$. This latter quantity is upper bounded by $t_{\textnormal{mix}}^{\bm{X}}+2$ where $t_{\textnormal{mix}}^{\bm{X}}$ denotes the mixing time of the chain $\bm{X}$ itself. Finally, the matrix $Q$ has eigenvalues $1$ and $\phi_2$, and an explicit computation yields that
	$\max_{ij} \abs{Q^n_{ij}-\pi_j} = \max_i (\pi_i) \abs{\phi_2}^n$ so that the mixing time of $\bm{X}$ is at most
	\[ \ceil[\Big]{\frac{\log 4}{\log(1/\abs{\phi_2})}}\leq \ceil[\Big]{\frac{\log 4}{\log(1/(1-L))}}\leq \ceil[\Big]{\frac{\log 4}{L}},\] which is a constant since $L$ is fixed. The pseudo-spectral gap of the chain $\bm{Z}$ is thus lower bounded by some constant $\gamma=\gamma(L)$.

	Applying \cref{eqn:MC-concentration} with $h(Z)=\II\braces{Z_4=a,Z_5=b,Z_6=c}$, which satisfies $\EE h^2 \leq 1$ and $\norm{h}_\infty \leq 1$, we see that
        \begin{multline*}
          \PP_{\phi,\psi}\brackets[\big]{ n \abs{\hat{p}^{(3)}(a,b,c)-p_{\phi,\psi}^{(3)}(a,b,c)}\geq x}\\
          \leq 2\exp \brackets[\Big]{-\frac{\gamma x^2 }{8n+8/\gamma +20x}},
        \end{multline*}
	hence for some constant $c'>0$
        \begin{multline*}
          \PP_{\phi,\psi}\brackets[\big]{ \abs{\hat{p}^{(3)}(a,b,c)-p_{\phi,\psi}^{(3)}(a,b,c)}\geq x/\sqrt{n}}\\
          \leq 2\exp\brackets[\Big]{-c' \min \brackets[\Big]{x^2,x^2 n,x\sqrt{n}}}.
        \end{multline*}
	Using that $\norm{\hat{p}^{(3)}-p^{(3)}}\leq K^3\max_{a,b,c}\abs{\hat{p}^{(3)}(a,b,c)-p^{(3)}(a,b,c)}$ and a union bound, we deduce for some $C=C(K,L)$ and for $x\leq \sqrt{n}$ that
	\[ \PP_{\phi,\psi}(\norm{\hat{p}^{(3)}-p^{(3)}}\geq K^3 x/\sqrt{n})\leq 2K^3\exp(-Cx^2).\]
	For $x\geq 1$ we may absorb the factor $2K^3$ into the exponential by changing the constant $C$, and by replacing $x$ with $C'x$ we can remove this constant, yielding the result in the case where $C'x \leq \sqrt{n}$. In the other case, since $\norm{\hat{p}^{(3)}-p^{(3)}}$ is bounded (by $K^{3/2}$), by increasing the constant $C'$ if necessary we have $C'x/\sqrt{n}\geq K^{3/2}$ so that the probability in question is equal to $0\leq e^{-x^2}$.
\end{proof}

To prove \cref{thm:minimax-upper-bound}, observe that by \cref{lem:estimation-of-p3} there exist events $\mathcal{A}_n$ of probability at least $e^{-x^2}$ on which
\[ \norm{\hat{p}^{(3)}_n-p_{\phi,\psi}^{(3)}}\leq Cx/\sqrt{n}.\]
The true parameter $(\phi,\psi)$ lies in $\Phi_L$ so that any estimators constructed in \cref{eqn:minimum-distance} satisfy
\[ \norm{p^{(3)}_{\hat{\phi},\hat{\psi}}-\hat{p}^{(3)}_n} \leq 2\norm{p^{(3)}_{\phi,\psi}-\hat{p}^{(3)}_n},\]
and hence on the event $\mathcal{A}_n$ further satisfy
\begin{align*}
  \norm{ p^{(3)}_{\hat{\phi},\hat{\psi}} - p^{(3)}_{\phi,\psi}}
  &\leq \norm{p^{(3)}_{\hat{\phi},\hat{\psi}}-\hat{p}^{(3)}_n}+\norm{p^{(3)}_{\phi,\psi}-\hat{p}^{(3)}_n}\\
  &\leq 3\norm{p^{(3)}_{\phi,\psi}-\hat{p}^{(3)}_n}\\
  &\leq 3Cx/\sqrt{n},
\end{align*}
By \cref{pro:local-equivalence} we deduce for a constant $C'$ that $\rho(\hat{\phi},\hat{\psi};\phi,\psi)\leq C'x/\sqrt{n}$ on $\mathcal{A}_n$.
For estimating $\psi$, observe that $\norm{\hat{\psi}_1-\psi_1}\leq \rho(\hat{\phi},\hat{\psi};\phi,\psi)$ and $\abs{r(\phi)}\min(\norm{\hat{\psi}_2-\psi_2},\norm{\hat{\psi}_2+\psi_2})\leq \rho(\hat{\phi},\hat{\psi};\phi,\psi)$. The upper bound for estimating $\psi_1$ is immediate and, recalling from \cref{eqn:phi-bounds} that $\abs{r(\phi)}\geq \delta\epsilon\zeta^2/4$, we also deduce the bound for $\psi_2$.

For the bounds on $\phi$, observe firstly that it suffices to prove the upper bounds on the absolute risk since, taking $\phi_2$ as an example, for $(\phi,\psi)\in \Phi_L(\delta,\epsilon,\zeta)$ we have
\begin{align}
  &\PP_{\phi,\psi}\brackets[\Big]{\abs{\hat{\phi}_2/\phi_2-1}^2\geq \frac{C}{n\delta^2\epsilon^4\zeta^4}}\nonumber\\
  &\qquad\qquad=  \PP_{\phi,\psi}\brackets[\Big]{\abs{\hat{\phi}_2-\phi_2}^2 \geq \frac{C\phi_2^2}{n\delta^2\epsilon^4\zeta^4}}\nonumber\\
  \label{eqn:absolute-upper-implies-relative}
  &\qquad\qquad\leq \PP_{\phi,\psi}\brackets[\Big]{\abs{\hat{\phi}_2-\phi_2}\geq \frac{C\epsilon^2}{n\delta^2\epsilon^4\zeta^4}}.
\end{align}
(See also after \cref{eqn:required-bounds-for-phi1} for a similar argument with $\phi_1$.)
Define
\begin{multline*}
	%	\label{eq:130}
	\omega_1(\phi,\psi;\eta)%
	\coloneqq\\ \sup\Set*{\abs{\phi_1 - \sgn(\Inner{\psi_2,\tilde{\psi}_2})\cdot \tilde{\phi}_1} \given \rho(\phi,\psi;\tilde{\phi},\tilde{\psi}) \leq \eta },
\end{multline*}
and for $j=2,3$
\begin{equation*}
	%\label{eq:37}
	\omega_j(\phi,\psi;\eta)%
	\coloneqq%
	\sup\Set*{\abs{\phi_j - \tilde{\phi}_j} \given \rho(\phi,\psi;\tilde{\phi},\tilde{\psi}) \leq \eta }.
\end{equation*}
We have the following.

\begin{proposition}
	\label{pro:modulus:pointwise}
	Let $\eta \in [0,1]$. There exist constants $c,C$ for which the following hold.
	\begin{align*}
		%	\label{eq:66}
		\eta <c (1-\phi_1^2)\phi_2^2\phi_3^3%
		&\implies%
		\omega_1(\phi,\psi;\eta)%
		\leq \frac{C \eta}{\phi_2^2\phi_3^3}, \\
		%\label{eq:67}
		\eta < c (1-\phi_1^2)\abs{\phi_2}\phi_3^2
		&\implies%
		\omega_2(\phi,\psi;\eta)%
		\leq \frac{C \eta}{(1-\phi_1^2)\abs{\phi_2}\phi_3^2}, \\
		%	\label{eq:40}
		\eta < c (1-\phi_1^2)\phi_2^2\phi_3^3
		&	\implies%
		\omega_3(\phi,\psi;\eta)%
		\leq \frac{C \eta}{(1 - \phi_1^2)\phi_2^2\phi_3^2}.
	\end{align*}
\end{proposition}
%Before proving the proposition we show how \cref{thm:minimax-upper-bound} follows from it.
The conditions of \cref{thm:minimax-upper-bound} ensure that on the event $\mathcal{A}_n$ we may apply \cref{pro:modulus:pointwise} with $\eta=C'x/\sqrt{n}$. We deduce the upper bounds for estimating the components of $\phi$ immediately upon replacing $\phi_1,\phi_2$ and $\phi_3$ on the right sides in \cref{pro:modulus:pointwise} by their lower bounds [for $\phi_1$ we note that $\min(\abs{\hat{\phi}_1-\phi_1},\abs{\hat{\phi}_1+\phi_1})\leq \abs{\phi_1-\sign(\ip{\psi_2,\tilde{\psi}_2})\cdot \tilde{\phi}_1}$].

\begin{proof}[Proof of \cref{pro:modulus:pointwise}]
	Recall that $m(\phi)=(r(\phi),\phi_2 r(\phi),\phi_1\phi_2\phi_3 r(\phi))$ with $r(\phi)=\tfrac{1}{4}(1-\phi_1^2)\phi_2\phi_3^2$.
	If $r(\phi)=0$ then in each case no $\eta\in[0,1]$ satisfies the conditions and so there is nothing to prove.
	Otherwise, note that $m$ is invertible when restricted to  $\braces{\phi: r(\phi)\neq 0}\supset \Phi(\delta,\epsilon,\zeta)$
	and its inverse is given by $\phi(m)$ defined by
	\begin{align*}
		%		\label{eq:38}
		\phi_1(m) &= m_3 / (4m_1^2m_2 + m_3^2)^{1/2} \\
		%	\label{eq:29}
		\phi_2(m) &= m_2/m_1, \\
		%	\label{eq:39}
		\phi_3(m) &= (4m_1^2m_2 + m_3^2)^{1/2}/m_2.
	\end{align*}
	For arbitrary $(\phi,\psi;\tilde{\phi},\tilde{\psi})$ satisfying $\rho(\phi,\psi;\tilde{\phi},\tilde{\psi}) \leq \eta$, we define
	\begin{gather*}
		%	\label{eq:122}
		\Delta_1 \coloneqq m_1(\tilde{\phi}) - m_1(\phi),\quad%
		\Delta_2 \coloneqq m_2(\tilde{\phi}) - m_2(\phi),\\%
		\Delta_3 \coloneqq \sgn(\Inner{\psi_2,\tilde{\psi}_2})\cdot m_3(\tilde{\phi}) - m_3(\phi).
	\end{gather*}
	Define also
	\begin{align*}
		%	\label{eq:123}
		g(\phi)
		&\coloneqq 4m_1(\phi)^2m_2(\phi) + m_3(\phi)^2\\
		&= \braces{m_2(\phi)\phi_3 }^2\\
		&= \braces[\Big]{\frac{1}{4}(1-\phi_1^2)\phi_2^2 \phi_3^3 }^2,
	\end{align*}
	and, for $\Delta=(\Delta_1,\Delta_2,\Delta_3)$,
	\begin{equation*}
		%	\label{eq:124}
		\begin{split}
			h_{\phi}(\Delta)
			&\coloneqq%
			g(\tilde{\phi}) - g(\phi)\\
			&= 4(m_1(\phi) + \Delta_1)^2(m_2(\phi) + \Delta_2) + (m_3(\phi) + \Delta_3)^2\\
                  &\quad- \braces{4m_1(\phi)^2m_2(\phi) + m_3(\phi)^2}.
		\end{split}
	\end{equation*}
	Observe that
	\begin{equation*}
		\begin{split}
			h_{\phi}(\Delta)%
			&= 8m_1(\phi)m_2(\phi)\Delta_1 + 8m_1(\phi)\Delta_1\Delta_2 + 4m_2(\phi)\Delta_1^2\\
			\label{eq:57}
			&\quad+ 4\Delta_1^2\Delta_2 + 4m_1(\phi)^2\Delta_2 + 2m_3(\phi)\Delta_3 + \Delta_3^2.
		\end{split}
	\end{equation*}

	\paragraph{Bounding $\omega_1$}

	We decompose,
        \begin{equation*}
          \phi_1 - \sgn(\Inner{\psi_2,\tilde{\psi}_2})\cdot \tilde{\phi}_1\\
          % = \frac{m_3(\phi)}{\sqrt{4m_1(\phi)^2m_2(\phi) +  m_3(\phi)^2}}%
	  % - \frac{\sgn(\Inner{\psi_2,\tilde{\psi}_2})\cdot m_3(\tilde{\phi})}{\sqrt{4m_1(\tilde{\phi})^2m_2(\tilde{\phi}) +  m_3(\tilde{\phi})^2}}
          = \frac{m_3(\phi)}{\sqrt{g(\phi)}} - \frac{m_3(\phi) + \Delta_3}{\sqrt{g(\phi) + h_{\phi}(\Delta)}}
        \end{equation*}
        which is in turn equal to
        \begin{equation*}
          m_3(\phi)\Big\{ \frac{1}{\sqrt{g(\phi)}} - \frac{1}{\sqrt{g(\phi) + h_{\phi}(\Delta)}} \Big\}%
		- \frac{\Delta_3}{\sqrt{g(\phi) + h_{\phi}(\Delta)}}
        \end{equation*}
        i.e.\ equal to
        \begin{multline*}
          \frac{m_3(\phi)}{\sqrt{g(\phi)(g(\phi) + h_{\phi}(\Delta))}}\big\{\sqrt{g(\phi) + h_{\phi}(\Delta)} - \sqrt{g(\phi)}  \big\}\\
          - \frac{\Delta_3}{\sqrt{g(\phi) + h_{\phi}(\Delta)}},
        \end{multline*}
        which is
        \begin{multline*}
          \frac{m_3(\phi)}{\sqrt{g(\phi)(g(\phi) + h_{\phi}(\Delta))}} \frac{h_{\phi}(\Delta)}{\sqrt{g(\phi)+h_{\phi}(\Delta)} + \sqrt{g(\phi)}}\\
          - \frac{\Delta_3}{\sqrt{g(\phi) + h_{\phi}(\Delta)}}.
        \end{multline*}
	Now we observe that $m_3(\phi)/\sqrt{g(\phi)}$ is equal to $\phi_1$, so indeed
	\begin{multline*}
		%	\label{eq:126}
		\phi_1 - \sgn(\Inner{\psi_2,\tilde{\psi}_2})\cdot \tilde{\phi}_1%
		=\\
                \frac{\phi_1 h_{\phi}(\Delta) - \Delta_3(\sqrt{g(\phi) + h_{\phi}(\Delta)} + \sqrt{g(\phi)}) }{\sqrt{g(\phi) + h_{\phi}(\Delta)}(\sqrt{g(\phi) + h_{\phi}(\Delta)} + \sqrt{g(\phi)}  )}.%
	\end{multline*}
	Call the numerator of this last fraction $N$ and call its denominator $D$. Writing $h_{\phi}(\Delta)$ as
	$h_{\phi}(\Delta) = \xi_{\phi}(\Delta) + \gamma_{\phi}(\Delta)$, where $\gamma_{\phi}(\Delta) \coloneqq 2m_3(\phi)\Delta_3 + \Delta_3^2$, we see that
	\begin{equation*}
		%	\label{eq:42}
		N%
		= \phi_1 \xi_{\phi}(\Delta) + \phi_1\gamma_{\phi}(\Delta)%
		- \Delta_3\big\{ (g(\phi) + h_{\phi}(\Delta))^{1/2} + g(\phi)^{1/2} \big\}.
	\end{equation*}
	In order to obtain the optimal upper bound, we need to do a fine analysis of this expression. To this end, we calculate
	\begin{align*}
		%	\label{eq:64}
		A%
		&\coloneqq \phi_1\gamma_{\phi}(\Delta) - \Delta_3\{ (g+h)^{1/2} + g^{1/2} \}\\
		&=2 \Delta_3\{\phi_1 m_3(\phi) - g^{1/2} \} +  \phi_1\Delta_3^2 - \Delta_3\{(g+h)^{1/2} - g^{1/2}\}\\
		&=-2\Delta_3(1-\phi_1^2)g^{1/2}%
		+ \phi_1 \Delta_3^2%
		- \frac{\Delta_3h}{(g+h)^{1/2} + g^{1/2}},
	\end{align*}
        i.e.\
        \begin{align*}
          A%
          &=-2\Delta_3(1-\phi_1^2)g^{1/2}\\
          &\quad
                  - \Delta_3 \frac{\gamma_{\phi}(\Delta) - \phi_1\Delta_3((g+h)^{1/2}+g^{1/2}  )}{(g+h)^{1/2} + g^{1/2}}\\%
          &\quad
		- \frac{\Delta_3 \xi_{\phi}(\Delta)}{(g+h)^{1/2} + g^{1/2}},
        \end{align*}
	where the last line follows because $\phi_1m_3(\phi) = \phi_1^2g(\phi)^{1/2}$.
	We now focus on the middle term of the last display, which we will express as a function of $A$.
	\begin{align*}
		%		\label{eq:69}
		B%
		&\coloneqq \gamma_{\phi}(\Delta) - \phi_1\Delta_3((g+h)^{1/2}+g^{1/2}  )\\%
		&= 2\Delta_3(m_3(\phi) - \phi_1g^{1/2}) + \Delta_3^2 - \phi_1\Delta_3 \{(g+h)^{ 1/2 } -g^{1/2} \}\\
		&= \Delta_3^2 - \frac{\phi_1 \Delta_3 h}{(g+h)^{1/2} + g^{1/2} }\\
		&= \Delta_3^2 - \frac{\phi_1 \Delta_3 \gamma_{\phi}(\Delta)}{(g+h)^{1/2} +g^{1/2} }%
		- \frac{\phi_1\Delta_3 \xi_{\phi}(\Delta)}{(g+h)^{1/2} + g^{1/2}}%
	  %       &= \Delta_3 \frac{\Delta_3\{(g+h)^{1/2} + g^{1/2} \} - \phi_1\gamma_{\phi}(\Delta) }{(g+h)^{1/2} + g^{1/2}}\\%
          % &\quad
	  %       - \frac{\phi_1\Delta_3 \xi_{\phi}(\Delta)}{(g+h)^{1/2} + g^{1/2}}\\
	\end{align*}
        that is,
        \begin{equation*}
          B%
          =- \frac{\Delta_3 A}{(g+h)^{1/2} + g^{1/2}}%
		- \frac{\phi_1\Delta_3 \xi_{\phi}(\Delta)}{(g+h)^{1/2} + g^{1/2}}.
        \end{equation*}
	Thus,
	\begin{align*}
		%	\label{eq:82}
		A%
		&= -2\Delta_3(1-\phi_1^2)g^{1/2}%
                  - \frac{\Delta_3 B}{(g+h)^{1/2} + g^{1/2}}\\%
          &\quad
		- \frac{\Delta_3 \xi_{\phi}(\Delta)}{(g+h)^{1/2} + g^{1/2}}\\
		&= -2\Delta_3(1-\phi_1^2)g^{1/2}%
                  + \frac{\Delta_3^2 A}{ \braces{(g+h)^{1/2}+g^{1/2}}^2 }\\%
          &\quad%
                  + \frac{\phi_1 \Delta_3^2 \xi_{\phi}(\Delta)}{\{(g+h)^{1/2} + g^{1/2}\}^2 }%
		- \frac{\Delta_3 \xi_{\phi}(\Delta)}{(g+h)^{1/2} + g^{1/2}},
	\end{align*}
	from which we deduce that
	\begin{multline*}
		%	\label{eq:91}
		N%
		= \phi_1\xi_{\phi}(\Delta)+ \\%
		\frac{-2\Delta_3(1-\phi_1^2)g^{1/2} + \frac{\phi_1 \Delta_3^2 \xi_{\phi}(\Delta)}{\braces{(g+h)^{1/2} + g^{1/2}}^2 }%
			- \frac{\Delta_3 \xi_{\phi}(\Delta)}{(g+h)^{1/2} + g^{1/2}} }{1 - \Delta_3^2/\braces{(g+h)^{1/2} + g^{1/2}}^2 }.
	\end{multline*}
	Since
	$m_2(\phi) \geq 0$, we see that $\xi_{\phi}(\Delta)$ has maximal amplitude when
	$\Delta_1 = \sgn(m_1(\phi))\eta$ and when $\Delta_2 = \eta$, in which case we have
	\begin{align*}
		%	\label{eq:47}
		\abs{\xi_{\phi}(\Delta)}%
          &= 8\abs{m_1(\phi)}m_2(\phi)\eta + 8\abs{m_1(\phi)}\eta^2\\
          &\quad+ 4 m_2(\phi)\eta^2 + 4\eta^3 + 4m_1(\phi)^2\eta\\
		&\leq 12 m_1(\phi)^2 \eta + 12 \abs{m_1(\phi)} \eta^2 + 4\eta^3,
	\end{align*}
	where the last line follows since $m_2(\phi) \leq \abs{m_1(\phi)}$. Now we observe that under
	the condition of the lemma, we have $\eta \lesssim \abs{m_1(\phi)}$, and so we can find a
	constant $C > 0$ such that
	\begin{equation*}
		%	\label{eq:50}
		\abs{\xi_{\phi}(\Delta)}%
		\leq C m_1(\phi)^2 \eta.
	\end{equation*}
	Also, we have that $\abs{\gamma_{\phi}(\Delta)} \leq 2\abs{m_3(\phi)}\eta + \eta^2$, and so
	\begin{equation*}
		%	\label{eq:52}
		\abs{h_{\phi}(\Delta)}%
		\leq Cm_1(\phi)^2 \eta + 2\abs{m_3(\phi)}\eta + \eta^2,
	\end{equation*}
	Noting that $\phi_3\leq \sqrt{K}$, for $c_0=c_0(K)$ sufficiently small in the assumption of the proposition we have $\abs{h_{\phi}(\Delta)} \leq g(\phi)/2$. Consequently, noting also that $\abs{\Delta_3}\leq \eta$ and $\eta \leq 4c_0 g^{1/2}$, we find that
	\begin{align*}
		%	\label{eq:53}
		\abs{N}%
		&\lesssim \abs{\phi_1}m_1(\phi)^2\eta%
		+ \eta(1-\phi_1^2) g(\phi)^{1/2}\\
		&\lesssim \eta (1 - \phi_1^2)^2 \phi_2^2\phi_3^3,
	\end{align*}
	and
	\begin{equation*}
		%	\label{eq:55}
		\abs{D}%
		\gtrsim g(\phi)%
		\gtrsim (1-\phi_1^2)^2\phi_2^4\phi_3^6.
	\end{equation*}
	Hence we have
	\begin{equation*}
		%	\label{eq:70}
		\abs{\phi_1 - \sgn(\Inner{\psi_2,\tilde{\psi}_2})\cdot \tilde{\phi}_1}%
		\lesssim \frac{\eta}{\phi_2^2\phi_3^3}.
	\end{equation*}
	\paragraph{Bounding $\omega_2$}

	We rewrite,
	\begin{align*}
		%	\label{eq:127}
		\phi_2 - \tilde{\phi}_2 %
		&= \frac{m_2(\phi)}{m_1(\phi)} - \frac{m_2(\phi) + \Delta_2}{m_1(\phi) + \Delta_1}\\
		&= \frac{m_2(\phi)(m_1(\phi) + \Delta_1) - (m_2(\phi) + \Delta_2)m_1(\phi)}{m_1(\phi)(m_1(\phi) + \Delta_1)}.
	\end{align*}
	Hence,
	\begin{equation*}
		%	\label{eq:51}
		\phi_2 - \tilde{\phi}_2%
		= \frac{\Delta_1 m_2(\phi) - \Delta_2 m_1(\phi)}{m_1(\phi)(m_1(\phi)+\Delta_1)}.
	\end{equation*}
	Under the assumptions of the theorem, we have that $\eta \leq m_1(\phi)/2$, and thus
	\begin{align*}
		%	\label{eq:71}
		\abs{\phi_2 - \tilde{\phi}_2}%
		&\leq \frac{2\eta(m_2(\phi) + |m_1(\phi)|)}{m_1(\phi)^2}\\
		&\leq \frac{4\eta}{|m_1(\phi)|}\\%
		&= \frac{16 \eta}{(1-\phi_1^2)|\phi_2|\phi_3^2}.
	\end{align*}

	\paragraph{Bounding $\omega_3$}

	We rewrite,
	\begin{align*}
		%	\label{eq:128}
		\phi_3 - \tilde{\phi}_3%
		&= \frac{\sqrt{g(\phi)}}{m_2(\phi)} - \frac{\sqrt{g(\phi) + h_{\phi}(\Delta)}}{m_2(\phi) + \Delta_2}\\
		&= \frac{m_2(\phi)(\sqrt{g(\phi)} - \sqrt{g(\phi) + h_{\phi}(\Delta)}) }{m_2(\phi)(m_2(\phi) + \Delta_2)}\\%
          &\quad
		+ \frac{\Delta_2 \sqrt{g(\phi)}}{m_2(\phi)(m_2(\phi) + \Delta_2)}\\
		&=\frac{-h_{\phi}(\Delta)}{(m_2(\phi) + \Delta_2)( \sqrt{g(\phi) + h_{\phi}(\Delta)} + \sqrt{g(\phi)} )}\\%
          &\quad%
		+ \frac{\Delta_2 \phi_3 }{m_2(\phi) + \Delta_2}\\
		&= \frac{-h_{\phi}(\Delta) + \Delta_2\phi_3(\sqrt{g(\phi) + h_{\phi}(\Delta)} + \sqrt{g(\phi)}) }{(m_2(\phi) + \Delta_2)( \sqrt{g(\phi) + h_{\phi}(\Delta)} + \sqrt{g(\phi)} )}
	\end{align*}
	Let us call the numerator of the fraction on the right of the last display  $N$, and the denominator $D$. We further decompose $h_{\phi}(\Delta)$ as
	$h_{\phi}(\Delta) = \xi_{\phi}(\Delta) + \gamma_{\phi}(\Delta)$, where $\gamma_{\phi}(\Delta) \coloneqq 4m_1(\phi)^2\Delta_2$.
	We see that
	\begin{align*}
		%	\label{eq:93}
		N%
		&= - \xi_{\phi}(\phi) - 4m_1(\phi)^2\Delta_2%
		+ \phi_3\Delta_2( (g+h)^{1/2} + g^{1/2}  )\\
		&= - \xi_{\phi}(\phi) -4 m_1(\phi)^2\Delta_2 + \phi_3\Delta_2\{ (g+h)^{1/2} + g^{1/2}  \}\\
		&=- \xi_{\phi}(\phi) - 4m_1(\phi)^2\Delta_2 + 2\phi_3\Delta_2 g^{1/2}\\
          &\quad+ \phi_3\Delta_2\{ (g+h)^{1/2} - g^{1/2} \}\\
		&=- \xi_{\phi}(\phi) +  \Delta_2 (1+\phi_1^2)\phi_3g^{1/2} + \frac{\phi_3\Delta_2 h}{(g+h)^{1/2} + g^{1/2}},
	\end{align*}
	where the last line follows because
	$m_1(\phi)^2 = \frac{1}{4}(1-\phi_1^2)\phi_3 g^{1/2}$. Since
	$m_2(\phi) \geq 0$, we see that $\xi_{\phi}(\Delta)$ has maximal amplitude when
	$\Delta_1 = \sgn(m_1(\phi))\eta$ and when $\Delta_2 = \eta$, in which case we have
	\begin{align*}
		%	\label{eq:102}
		\abs{\xi_{\phi}(\Delta)}%
          &= 8\abs{m_1(\phi)}m_2(\phi)\eta + 8\abs{m_1(\phi)}\eta^2 + 4m_2(\phi)\eta^2\\
          &\quad+ 4\eta^3%
		+2 \abs{m_3(\phi)}\eta + \eta^2\\
		%		&\leq \{15|m_1(\phi)|m_2(\phi) + 2|m_3(\phi) \}\eta + \eta^2\\
		&\lesssim \braces{\abs{m_1(\phi)}m_2(\phi) + 2\abs{m_3(\phi) }}\eta + \eta^2\\
		&\lesssim (1-\phi_1^2)\phi_2^2\phi_3^3 \max\braces{ (1-\phi_1^2) \abs{\phi_2}\phi_3,\, \abs{\phi_1}}\eta + \eta^2,
	\end{align*}
	where the second line follows because under the assumptions of the proposition we have that $m_2(\phi) \lesssim \abs{m_1(\phi)}$ and $\eta \leq m_2(\phi)/2$ (note that $\phi_3\leq K^{1/2}$). Since
	$h_{\phi}(\Delta) = \xi_{\phi}(\Delta) + 4m_1(\phi)^2\Delta_2$, we also have
	\begin{equation*}
		%	\label{eq:104}
		\abs{h_{\phi}(\Delta)}%
		\lesssim%
		(1-\phi_1^2)\phi_2^2\phi_3^3 \max\braces{ (1-\phi_1^2)\phi_3,\, \abs{\phi_1}}\eta + \eta^2,
	\end{equation*}

	Hence,
	\begin{align*}
		%	\label{eq:103}
		\abs{N}%
		&\lesssim%
		(1-\phi_1^2)\phi_2^2\phi_3^3 \max\braces{ (1-\phi_1^2) \abs{\phi_2}\phi_3,\, \abs{\phi_1}}\eta + \eta^2\\
		&\quad + \eta \phi_3 g^{1/2}\\%
          &\quad%
		+ \frac{\eta^2\phi_3(1-\phi_1^2)\phi_2^2\phi_3^3\max\braces{(1-\phi_1^2)\phi_3,\abs{\phi_1}} + \eta^3\phi_3 }{g^{1/2}}\\
		& \lesssim%
		(1-\phi_1^2)\phi_2^2\phi_3^3 \max\braces{ (1-\phi_1^2) \abs{\phi_2}\phi_3,\, \abs{\phi_1}}\eta + \eta^2\\
		&\quad + \eta \phi_3 g^{1/2}%
                  + \eta^2\phi_3\max\braces{(1-\phi_1^2)\phi_3,\abs{\phi_1}}\\%
          &\quad%
		+ \frac{\eta^3}{(1-\phi_1^2)\phi_2^2\phi_3^2}
	\end{align*}
	But by assumption $\eta \lesssim (1-\phi_1^2)\phi_2^2\phi_3^3$, and $4g^{1/2} = (1-\phi_1^2)\phi_2^2\phi_3^2$, thus
	\begin{equation*}
		%	\label{eq:108}
		\abs{N}%
		\lesssim
		(1-\phi_1^2)\phi_2^2\phi_3^3 \max\braces{\phi_3,\, \abs{\phi_1}}\eta + \eta^2.
	\end{equation*}
	Note that $\max\brackets{\phi_3,\abs{\phi}_1}\leq \sqrt{K}$.	Moreover, under the assumptions of the proposition and using that $\phi_3\leq \sqrt{K}$, it is
	the case that $\abs{\Delta_2} \leq \eta \lesssim m_2(\phi)$.
	Therefore $\abs{D} \gtrsim m_2(\phi) \sqrt{g(\phi)}$, and
	\begin{equation*}
		%	\label{eq:41}
		\abs{\phi_3 - \tilde{\phi}_3}%
		\lesssim \frac{\eta }{(1 - \phi_1^2)\phi_2^2\phi_3^2}%
		+ \frac{\eta^2}{(1-\phi_1^2)^2\phi_2^4\phi_3^5}.
	\end{equation*}
	Finally, since have assumed that $\eta<\frac{(1-\phi_1^2)\phi_2^2\phi_3^2}{8}$, we see that the second term is at most a constant times the first, so that it can be absorbed by increasing the constant $C$.
\end{proof}

\subsection{Proof of \cref{thm:mlb}}
\label{sec:proof-thm:mlb}
We give a standard two-point testing lower bound, summarising ideas that can be found for example in Chapter 2 of \cite{tsybakov:2009}.
\begin{lemma}\label{lem:2-point-testing-bound}
	Given data $X^{(n)}\sim p_u^{(n)}$ for parameter $u\in\mathcal{U}$, the following lower bounds hold for estimating $u$.

	Suppose $\mathcal{U}\subseteq \RR$ and for some $r\leq 1/2$ assume that there exist parameters $u_0,u_1$ satisfying
	\begin{enumerate}[i.]
		\item $\abs{u_1/u_0-1}\geq 4r$,
		\item $\KL(p^{(n)}_{u_1};p^{(n)}_{u_0})\leq 1/100$,
	\end{enumerate}
	where we recall $\KL$ denotes the Kullback--Leibler divergence. % $\KL(p_1,p_0)=\EE_{p_1}[\log(p_1/p_0)]$ denotes the Kullback--Leibler divergence between distributions with densities $p_1,p_0$.
	Then
	\[\inf_{\hat{u}}\sup_{u\in \mathcal{U}} \PP_u( \abs{\hat{u}/u-1}\geq r) \geq 1/4,\]
	where the infimum is over all estimators $\hat{u}$ based on the data $X^{(n)}$.

	If instead $(\mathcal{U},d)$ is a pseudo-metric space and for some $r\geq 0$ there exist parameters $u_0,u_1$ satisfying
	\begin{enumerate}[i.]
		\item $d(u_0,u_1)\geq 2r$
		\item $\KL(p^{(n)}_{u_1},p^{(n)}_{u_0})\leq 1/100$,
	\end{enumerate}
	then
	\[\inf_{\hat{u}}\sup_{u\in \mathcal{U}} \PP_u( d(\hat{u},u)\geq r)\geq 1/4.\]
\end{lemma}
\begin{proof}
In the case $\mathcal{U}\subseteq \RR$,	given an estimator $\hat{u}$ we may construct a test $T$ of $u=u_0$ vs $u=u_1$,
	\[ T=\II\braces[\Big]{\abs[\Big]{\frac{\hat{u}}{u_0}-1}>\abs[\Big]{\frac{\hat{u}}{u_1}-1}}.\]
	Observe that
	\begin{align*}
		\abs[\Big]{\frac{\hat{u}}{u_0}-1} &= \abs[\Big]{\frac{u_1}{u_0}-1 + \frac{\hat{u}-u_1}{u_1}\frac{u_1}{u_0}} \\
		&\geq 4r - \abs[\big]{\frac{\hat{u}}{u_1}-1}(1+4r).
	\end{align*}
	Then \begin{align*}
		\PP_{u_1}\brackets{T=0}&=\PP_{u_1}\brackets[\Big]{\abs[\Big]{\frac{\hat{u}}{u_0}-1}\leq\abs[\Big]{\frac{\hat{u}}{u_1}-1}} \\
		&\leq \PP_{u_1}\brackets[\Big]{4r-\abs[\big]{\frac{\hat{u}}{u_1}-1}(1+4r)\leq \abs[\Big]{\frac{\hat{u}}{u_1}-1}} \\
		&\leq  \PP_{u_1}\brackets[\Big]{\abs[\Big]{\frac{\hat{u}}{u_1}-1}\geq r},
	\end{align*}
	where for the last line we have used that $4r/(2+4r)\geq r$ for $r\leq 1/2$.
	Also note that on the event $\braces{T=1}\cap\braces{\abs{\hat{u}/u_0-1}<r}$ we have also $\abs{\hat{u}/u_1-1}<r$ and hence
	\begin{align*}
          &\abs{u_1/u_0-1}\\
          &\qquad=\abs{\hat{u}/u_0-1-(\hat{u}/u_1-1) - (\hat{u}/u_1-1)(u_1/u_0-1)} \\
		&\qquad< 2r+r\abs{u_1/u_0-1},
	\end{align*}
	so that $\abs{u_1/u_0-1}<2r/(1-r)$ on this event. Having assumed $r\leq 1/2$ and $\abs{u_1/u_0-1}\geq 4r$ we deduce that $\braces{T=1}\cap\braces{\abs{\hat{u}/u_0-1}<r}=\emptyset$ so that $\braces{T=1}\subseteq \braces{\abs{\hat{u}/u_0-1}\geq r}$, and hence we have shown
        \begin{align*}
          \inf_{\hat{u}}\sup_{u} \PP_u\brackets[\Big]{\abs[\Big]{\frac{\hat{u}}{u}-1}\geq r}%
          &\geq \inf_{\hat{u}} \max_{i=0,1} \PP_{u_i}\brackets[\Big]{\abs[\Big]{\frac{\hat{u}}{u_i}-1}\geq r}\\
          &\geq \inf_T \max_{i=0,1} \PP_{u_i}\brackets{T\neq i},
        \end{align*}
	where the latter infimum is over all tests $T$.
	In the pseudo-metric case a reduction considering the test $T=\II\braces{d(\hat{u},u_0)>d(\hat{u},u_1)}$ and directly using the triangle inequality likewise yields
	\[ \inf_{\hat{u}}\sup_{u} \PP_u\brackets[\big]{d(\hat{u},u)\geq r} \geq \inf_T \max_{i=0,1} \PP_{u_i}\brackets{T\neq i}.\]

	It remains to lower bound the maximum probability of testing error by $1/4$. Introducing the event $A=\braces[\big]{\frac{p_{u_0}^{(n)}}{p_{u_1}^{(n)}}\geq 1/2}$, we see
	\begin{equation*}
		\PP_{u_0}(T\not = 0) \geq  \EE_{u_1}\sqbrackets[\big]{\tfrac{p_{u_0}^{(n)}}{p_{u_1}^{(n)}} \II_A T}
		\geq \tfrac{1}{2}\sqbrackets{ \PP_{u_1}\brackets{T=1}- \PP_{u_1}(A^c)}
	\end{equation*}
	Thus, writing $p_1=\PP_{u_1}(T=1)$, we see
	\begin{align*}
          &\max\brackets{\PP_{u_0}(T\not=0),\PP_{u_1}(T\not =1)}\\
          &\qquad
            \geq \max \brackets{ \tfrac{1}{2}(p_1-\PP_{u_1}(A^c)),1-p_1} \\
          &\qquad\geq \inf_{p\in[0,1]} \max \brackets{ \tfrac{1}{2}(p-\PP_{u_1}(A^c)),1-p}.
        \end{align*}
	The infimum is attained when $\frac{1}{2}(p-\PP_{u_1}(A^c))=1-p$ and takes the value $\frac{1}{3} \PP_{u_1}(A)$, so that
	\begin{equation*}
		\inf_T \max_{i=0,1} \PP_{u_i}(T\neq i) \geq \tfrac{1}{3} \PP_{u_1}(A).
	\end{equation*}
	Next observe
	\begin{align*}
          \PP_{u_1}(A)&= \PP_{u_1}\sqbrackets[\big]{\tfrac{p_{u_1}^{(n)}}{p_{u_0}^{(n)}} \leq  2}\\
          &= 1- \PP_{u_1}^n\sqbrackets[\big]{\log\brackets[\big]{\tfrac{p_{u_1}^{(n)}}{p_{u_0}^{(n)}}}
            >\log 2}\\
          &\geq 1- \PP_{\theta_1}^n\sqbrackets[\big]{\abs{\log(\tfrac{p_{u_1}^{(n)}}{p_{u_0}^{(n)}})}>\log 2} \\ &\geq 1-(\log{2})^{-1} \EE_{u_1} \abs[\big]{\log \brackets[\big]{ \tfrac{p_{u_1}^{(n)}}{p_{u_0}^{(n)}}}},
	\end{align*}
	where we have used Markov's inequality to attain the final expression.
	By the second Pinsker inequality (e.g.\ Proposition 6.1.7b in \cite{GN16}), using the upper bound on the Kullback--Leibler divergence we can continue the chain of inequalities to see
        \begin{align*}
          \PP_{u_1}(A)
          &\geq 1-(\log 2)^{-1}\sqbrackets[\big]{\KL(p_{u_1}^{(n)},p_{u_0}^{(n)}) + \sqrt{2\KL(p_{u_1}^{(n)},p_{u_0}^{(n)})} }\\
          &\geq 1- (\log 2)^{-1} \brackets{\mu +\sqrt{2\mu}}.
        \end{align*}
	For any $c<1/3$, we may choose $\mu=\mu(c)$ small enough that the testing error satisfies \[\inf_T \max_{i=0,1} \PP_{u_i}(T\neq i)\geq \tfrac{1}{3} \brackets[\Big]{1- \frac{\mu+\sqrt{2\mu}}{\log 2}} >c,\] and in particular a numerical calculation shows that $\mu=1+\tfrac{1}{4}\log 2 - \sqrt{1+\tfrac{1}{2}\log 2}>1/100$ works for $c=1/4$.
\end{proof}

In view of  %\cref{pro:local-equivalence,kul:upper},
\cref{kul:upper}, for any $(\phi,\psi),(\tilde{\phi},\tilde{\psi})\in\Phi$ corresponding to strictly positive emission densities, we have for $\phi_2$ and $\tilde{\phi}_2$ small enough that
\[ \KL(p_{\phi,\psi}^{(n)},p_{\tilde{\phi},\tilde{\psi}}^{(n)})\leq C n \rho(\phi,\psi;\tilde{\phi},\tilde{\psi})^2,\]
where $C>0$ is a constant depending only on $K$ and a lower bound for the emission densities. We remark that for all the hypotheses we will exhibit below, we will have that $\phi_2$ and $\tilde{\phi}_2$ are of order $\epsilon$, which is upper bounded by $\epsilon_1$ by assumption, so that choosing the latter small enough the above bound on $\KL(p_{\phi,\psi}^{(n)},p_{\tilde{\phi},\tilde{\psi}}^{(n)})$ will apply. Then, to prove \cref{item:mlb:phi1}, it suffices to apply \cref{lem:2-point-testing-bound} to $u=1-\phi_1^2$ and prove the existence of parameters $(\phi,\psi),(\tilde{\phi},\tilde{\psi})\in \Phi_L(\delta,\epsilon,\zeta)$ satisfying for small enough $c_1> 0$ and some $c_2 > 0$
\begin{equation} \label{eqn:required-bounds-for-phi1}\rho(\phi,\psi;\tilde{\phi},\tilde{\psi})\leq c_1/\sqrt{n},\ \text{and} \ %
  \abs[\Big]{\frac{1-\tilde{\phi}_1^2}{1-\phi_1^2}-1}\geq c_2/\sqrt{n\delta^2\epsilon^4\zeta^6}
\end{equation}
which will give the lower bound on the absolute risk. Regarding the relative risk, we then note that for any $a\geq 0$, since $\abs{\phi_1}\leq 1$ and $1-\phi_1^2\geq \delta,$ so that we may assume the same of $\hat{\phi}_1$, we have
\begin{align*}
  &\PP_{\phi,\psi}(\abs{\hat{\phi}_1-\phi_1}\wedge \abs{ \hat{\phi}_1+\phi_1 }\geq a)\\
  &\qquad\geq 	\PP_{\phi,\psi}(\abs{(1-\hat{\phi}_1^2)-(1-\phi_1^2)}\geq 2a) \\ &\qquad\geq \PP_{\phi,\psi}(\abs{(1-\hat{\phi}_1^2)/(1-\phi_1^2)-1}\geq 2a/\delta).
\end{align*}
(See also \cref{eqn:absolute-upper-implies-relative} for a similar calculation with $\phi_2$.)

Similar conditions to \eqref{eqn:required-bounds-for-phi1} suffice for proving the other parts of \cref{thm:mlb} and we proceed now to verifying the existence of suitable parameters $(\phi,\psi)$ and $(\tilde{\phi},\tilde{\psi})$, with the help of the following lemma.

\begin{lemma}\label{lem:exists-psi}
	For a given $\phi$, assume conditions \eqref{eqn:condition-on-phi2} and \eqref{eqn:condition-on-phi} and assume that $\phi_3\leq \sqrt{2 \floor{K/2}}/(2K)$.
	Then there exists $\psi$ such that $(\phi,\psi)$ lies in $\Phi_L$ and the corresponding emission densities $f_0,f_1$ are bounded below by some constant $c=c(K)>0$. %away from zero. Moreover, the lower bound on the emission densities can be chosen uniformly in such $\phi$.

	In particular, for $\abs{\phi_1}\leq 1-3\delta$, $\epsilon\leq \phi_2\leq \min(1/3,1-L)$, $\zeta\leq \phi_3\leq 2\zeta$, such a $\psi$ exists under the condition $\zeta\leq \sqrt{2 \floor{K/2}}/(4K)$. %\eqref{eqn:compatibility}.
\end{lemma}
\begin{proof}
  For $k\leq K$, set $\psi_1(k)=1/K$ and 
  \begin{equation*}
    \psi_2(k)=(2\floor{K/2})^{-1/2} (\II\braces{k\textnormal{ odd, } k<K}-\II\braces{k\textnormal{ even}}).
  \end{equation*}
  [Or, similarly, $\psi_2(k)= (2\floor{K/2})^{-1/2} (\II\braces{k<(K+1)/2}-\II\braces{k>(K+1)/2})$.] Under the assumed condition on $\phi_3$ and recalling that $\abs{\phi_1}\leq 1$ by assumption, we observe from the expressions for $f_0,f_1$ given in \cref{rem:invert-param} that these are lower bounded by $1/(2K)$.
	In the particular case, one simply notes that all the conditions hold for such $\phi$.
\end{proof}

\paragraph{Proof of \cref{item:mlb:phi1,item:mlb:phi3}}
We prove the lower bounds for estimating $\phi_1$ and $\phi_3$ together.
For some small constant $c>0$, set $R=c\epsilon^{-2}\zeta^{-3}n^{-1/2}$ and, writing $S=(2-6\delta-R)R/(6\delta-9\delta^2)$, set
\begin{align*} \phi&=(1-3\delta,\epsilon,\zeta \sqrt{1+S}),\\
	\tilde{\phi}&=(1-3\delta-R,\epsilon,\zeta).
\end{align*}
Recalling the definition $r(\phi)=(1-\phi_1^2)\phi_2\phi_3^2/4$, the choice of $\phi_3$ ensures that $r(\phi)=r(\tilde{\phi})$, and we note that under the assumptions of the theorem we have $R\leq \delta\leq 1/6$ so that $S\leq R/\delta\leq 1$ and $\zeta\leq \phi_3\leq 2\zeta$. By \cref{lem:exists-psi} there exists some $\psi=\tilde{\psi}$ such that $(\phi,\psi),(\tilde{\phi},\tilde{\psi})\in\Phi_L$ and for this $\psi=\tilde{\psi}$ we see that
\begin{align*}
  \rho(\phi,\psi;\tilde{\phi},\tilde{\psi})
  &=\abs{\phi_1\phi_2\phi_3 r(\phi) - \tilde{\phi}_1\tilde{\phi}_2\tilde{\phi}_3 r(\tilde{\phi})}\\
  &= \phi_2 r(\phi)\abs{\phi_1\phi_3-\tilde{\phi}_1\tilde{\phi}_3}.
\end{align*}
Using that $\sqrt{1+t}\leq 1+t$ for $t\geq 0$ we have
\begin{align*}
  \abs{\phi_1\phi_3-\tilde{\phi}_1\tilde{\phi}_3}
  &= (1-3\delta)\zeta(\sqrt{1+S}-1)+R\zeta\\
  &\leq (S+R)\zeta\leq 2R\zeta/\delta,
\end{align*}
hence since $r(\phi)=(6\delta-9\delta^2)\epsilon \zeta^2(1+S)/4 \leq 3\delta \epsilon \zeta^2$, we obtain
\[\rho(\phi,\psi;\tilde{\phi},\tilde{\psi})\leq 6\epsilon^2 \zeta^3 R \leq 6cn^{-1/2}.\]
Recalling that $R\leq \delta\leq 1/6$ and that $r(\phi)=r(\tilde{\phi})$ one calculates
\[ \frac{1-\tilde{\phi}_1^2}{1-\phi_1^2}-1 = S\geq R/(12\delta).\]
For $c$ small enough we see that the conditions in \cref{eqn:required-bounds-for-phi1} are satisfied, yielding the claimed bound for estimating $\phi_1$.

To prove the lower bound for estimating $\phi_3$ it suffices to lower bound $\abs{\phi_3/\tilde{\phi}_3-1}$. Here we use the bound $\sqrt{1+x}-1\geq x/(2\sqrt{1+x})\geq x/(2\sqrt{2})$ for $0\leq x\leq 1$ to see for a constant $c'>0$ that
\[ \abs{\phi_3/\tilde{\phi}_3-1}\geq c'R/\delta.\] The bound for $\phi_3$ follows from applying \cref{lem:2-point-testing-bound}.

\paragraph{Proof of \cref{item:mlb:phi2}}
For a constant $c>0$, define
\(R=c\delta^{-1} \epsilon^{-1}\zeta^{-2}n^{-1/2},\) define $\phi,\tilde{\phi}$ by
\begin{align*}%\label{eq:75c}
	\phi&=(1-3\delta,\epsilon,\zeta(1+R/\epsilon)^{1/2})\\
	\tilde{\phi}&=(1-3\delta,\epsilon+R,\zeta),
\end{align*}
and observe that by construction $r(\phi)=r(\tilde{\phi})$.
Noting that $\phi_2\leq 2\epsilon\leq 1-L$ and $\phi_3\leq 2\zeta$ because the assumptions of \cref{thm:mlb} ensure that $R\leq \epsilon\leq 1/3$, we deduce using \cref{lem:exists-psi} that there exists some $\psi=\tilde{\psi}$ such that $(\phi,\psi),(\tilde{\phi},\tilde{\psi})\in \Phi_L(\delta,\epsilon,\zeta)$.% as with the previous cases.

Next observe, using that $(1+x)^{1/2}\leq 1+x$,
\begin{align*}
  \phi_1\abs{\phi_2\phi_3-\tilde{\phi}_2\tilde{\phi}_3}
  &\leq \abs{\phi_2}\abs{\phi_3-\tilde{\phi}_3} + \abs{\tilde{\phi}_3}\abs{\phi_2-\tilde{\phi}_2}\\ &=\epsilon\zeta(\sqrt{1+R/\epsilon}-1) + \zeta R\\
  &\leq 2\zeta R\\
  &\leq R,
\end{align*}
the last inequality holding if $\zeta\leq \zeta_0\leq 1/2$. We deduce
\begin{align*}
  \rho(\phi,\psi; \tilde{\phi},\tilde{\psi})
  &=r(\phi)\max( \abs{\phi_2- \tilde{\phi}_2},\abs{\phi_1\phi_2\phi_3-\tilde{\phi}_1\tilde{\phi}_2\tilde{\phi}_3})\\
  &= R r(\phi).
\end{align*} Again using that $\phi_3\leq 2\zeta$ and noting also that $(1-\phi_1^2)= 6\delta-9 \delta^2 \leq 6\delta$, we see that for some $C'>0$ we have
\begin{equation*}
	\rho(\phi,\psi;\tilde{\phi},\tilde{\psi}) \leq C' \delta \epsilon \zeta^2 R\leq cC'n^{-1/2}.
\end{equation*}
As with \cref{item:mlb:phi1,item:mlb:phi3}, for $c$ small enough in the definition of $R$ we may apply \cref{lem:2-point-testing-bound} to deduce the claimed lower bound since $\abs{\tilde{\phi}_2/\phi_2-1}=R/\epsilon$.

\paragraph{Proof of \cref{item:mlb:psi1}}
Set $\phi=\tilde{\phi}=(0,\epsilon,\zeta)$. For $k\leq K$, as in \cref{lem:exists-psi} define $\psi_1(k)=1/K$ and
\begin{equation*}
  \psi_2(k)=(2\floor{K/2})^{-1/2} (\II\braces{k\textnormal{ odd, } k<K}-\II\braces{k\textnormal{ even}}),
\end{equation*}
and set $\tilde{\psi}_1=\psi_1+cn^{-1/2}\psi_2$. Note that for the upper bound $\zeta_0$ small enough we have $(\phi,\psi),(\tilde{\phi},\tilde{\psi})\in\Phi_L$ for $n$ larger than some $C=C(K,c)$, or for all $n\geq 1$ if $c$ is small enough.
Then
\[ \rho(\phi,\psi;\tilde{\phi},\tilde{\psi})= cn^{-1/2}\] and we apply \cref{lem:2-point-testing-bound} to deduce the result.

\paragraph{Proof of \cref{item:mlb:psi2}}
Set $\phi=\tilde{\phi}=(1-3\delta,\epsilon,\zeta)$, choose $\psi_1=\tilde{\psi}_1$ to be the uniform density on $\braces{1,\dots,K}$. As with the previous parts, an application of \cref{lem:2-point-testing-bound} will yield the theorem if we can exhibit $\psi_2,\tilde{\psi}_2$ such that the induced emission densities are bounded below by some $c'=c'(K)>0$, $\norm{\psi_2-\tilde{\psi}_2}=R:=c(n\delta^2\epsilon^2\zeta^4)^{-1/2}$ for some $c>0$, $\sign(\ip{\psi_2,\tilde{\psi}_2})=+1$, and $\rho(\phi,\psi;\tilde{\phi},\tilde{\psi})\leq c_1n^{-1/2}$ for a small  constant $c_1$.
Such a choice is possible if the constants $c$ and $\zeta_0$ are small enough constants, for $n\delta^2\epsilon^2\zeta^4\geq 1$ with $\zeta\leq \zeta_0$; for example, define $\psi_2(k),~k\leq K$ as in \cref{lem:exists-psi} by \[\psi_2(k)=(2\floor{K/2})^{-1/2} (\II\braces{k\textnormal{ odd, } k<K}-\II\braces{k\textnormal{ even}}),\]
and, for $h$ defined by $h(1)=2^{-1/2}$, $h(3)=-2^{-1/2}$ and $h(k)=0$ for all other $k$, define
\[ \tilde{\psi}_2= (\psi_2+\alpha h)/(1+\alpha), \quad \alpha= R/(2-R).\]
This satisfies $\norm{\tilde{\psi}_2-\psi_2}=R$, $\norm{\tilde{\psi}_2}=1$, $\ip{\tilde{\psi}_2,1}=0$ and $\ip{\tilde{\psi}_2,\psi_2}\geq 0$. For $k\not \in\braces{1,3}$ the condition \eqref{eqn:condition-on-psi} of \cref{rem:valid-phis} holds with $1/(2K)$ in place of $0$ on the right, and for $k\in\braces{1,3}$ a direct calculation shows that the condition with $1/(4K)$ in on the right if $R$ is upper bounded by some $c'=c'(K)$, %holds if $R\leq 2\sqrt{2}/(4K+\sqrt{2\floor{K/2}})$, or so
which is the case for $c=c(K)$ sufficiently small. Then
\[\rho(\phi,\psi;\tilde{\phi},\tilde{\psi})=\abs{r(\phi)}\norm{\psi_2-\tilde{\psi}_2}\leq \delta\epsilon\zeta^2R\leq cn^{-1/2}.\]

\subsection{Proofs for \cref{sec:main-results}}
\begin{proof}[Proof of \cref{cor:rates-for-theta}]
	We begin with the upper bounds. From the inversion formulae in \cref{rem:invert-param} we have
	\begin{align*}
          \norm{\hat{f}_0-f_0} \vee \norm{\hat{f}_1-f_1}%
          &\leq \norm{\hat{\psi}_1-\psi_1}\\
          &\quad+\tfrac{1}{2}\norm{\hat{\phi}_1\hat{\phi}_3 \hat{\psi}_2-\phi_1\phi_3\psi_2}\\
          &\quad+\tfrac{1}{2}\norm{\hat{\phi}_3\hat{\psi}_2-\phi_3\psi_2}.
	\end{align*}
	Recalling that $\abs{\hat{\phi}_1}\leq 1$, that $0\leq \phi_3\leq K^{1/2}$ and that $\norm{\psi_2}=\norm{\hat{\psi}_2}=1$, we decompose the second term on the right, with an implicit decomposition of the third term included:
	\begin{align*}
          &\norm{\hat{\phi}_1\hat{\phi}_3\hat{\psi}_2-\phi_1\phi_3\psi_2}\\
          &\qquad\qquad \leq \abs{\hat{\phi}_1}\norm{\hat{\phi}_3\hat{\psi}_2-\phi_3\psi_2}+\abs{\phi_3}\abs{\hat{\phi}_1-\phi_1}\\
		&\qquad\qquad \leq \abs{\hat{\phi}_3-\phi_3}+K^{1/2}\norm{\hat{\psi}_2-\psi_2} + \phi_3 \abs{\hat{\phi}_1-\phi_1}.
	\end{align*}
	It follows that for some constant $C$ we have
        \begin{multline*}
          \norm{\hat{f}_0-f_0} \vee\norm{\hat{f}_1-f_1}\\
          \leq C\max(\norm{\hat{\psi}_1-\psi_1},\norm{\hat{\psi}_2-\psi_2},\abs{\hat{\phi}_3-\phi_3},\phi_3\abs{\hat{\phi}_1-\phi_1}).
        \end{multline*}
	Applying \cref{pro:modulus:pointwise} as in the proof of \cref{thm:minimax-upper-bound}, one can show that for some $C>0$
	\[ \PP_{\phi,\psi}\brackets[\Big]{\phi_3^2\abs{\hat{\phi}_1-\phi_1}^2\geq \frac{Cx^2}{n\epsilon^4\zeta^4}} \leq e^{-x^2}.\] The upper bounds for estimating $f_0$ and $f_1$ then follow from \cref{thm:minimax-upper-bound}.

	Similarly, \cref{rem:invert-param}
	and the fact that $\abs{\phi_2}\leq 1$ give
	\begin{align*}
          \abs{\hat{p}-p} \vee \abs{\hat{q}-q}%
          &\leq  \tfrac{1}{2}(1+\abs{\hat{\phi}_1})\abs{\hat{\phi}_2-\phi_2}+\tfrac{1}{2}\abs{\hat{\phi}_1-\phi_1}\abs{1-\phi_2}\\
          &\leq 2(\abs{\hat{\phi}_1-\phi_1} \vee \abs{\hat{\phi}_2-\phi_2}).
        \end{align*}
	The upper bounds then again follow from \cref{thm:minimax-upper-bound}.

	For the lower bounds, writing $\theta(\phi,\psi)=(p,q,f_0,f_1)$ and $\theta(\tilde{\phi},\tilde{\psi})=(\tilde{p},\tilde{q},\tilde{f}_0,\tilde{f}_1)$, observe by \cref{lem:2-point-testing-bound} that it suffices to lower bound $\max(\abs{p-\tilde{p}},\abs{q-\tilde{q}})$ and $\max(\norm{f_0-\tilde{f}_0},\norm{f_1-\tilde{f}_1})$ corresponding to choices of $(\phi,\psi),(\tilde{\phi},\tilde{\psi})$ made in the proof of \cref{thm:mlb}.

	From the inversion formulae in \cref{rem:invert-param} we calculate, for any $\phi,\tilde{\phi}$,
	\begin{multline}
          \label{eqn:max-pq}
          2(\abs{p-\tilde{p}} \vee \abs{q-\tilde{q}})% \\
          \geq\\
          \max( (1+\abs{\phi_1})\abs{\phi_2-\tilde{\phi}_2}
          -\abs{1-\tilde{\phi}_2}\abs{\phi_1-\tilde{\phi}_1},\\
          \abs{\phi_1-\tilde{\phi}_1}\abs{1-\tilde{\phi}_2}-(1-\abs{\phi_1})\abs{\phi_2-\tilde{\phi}_2}).
	\end{multline}
	If $\delta>\epsilon \zeta$ set
	$\phi= (1-3\delta,\epsilon,\zeta(1+S)^{1/2})$ and $\tilde{\phi}=(1-3\delta-R,\epsilon,\zeta)$, where $R=c(n\epsilon^4\zeta^6)^{-1/2}$ for some $c>0$ and where $S\in[R/(12\delta),R/\delta]$ is, as in the proof of \cref{thm:mlb} \cref{item:mlb:phi1}, such that $r(\phi)=r(\tilde{\phi})$.  If $\delta\leq \epsilon \zeta$ instead set
	$\phi= (1-3\delta,\epsilon,\zeta(1+R/\epsilon)^{1/2}),$ $\tilde{\phi}=(1-3\delta,\epsilon+R,\zeta)$ with $R=c(n\epsilon^2\delta^2\zeta^4)^{-1/2}$. In either case the proof of \cref{thm:mlb} demonstrates that for suitable $\psi=\tilde{\psi}$ we have $\KL(p^{(n)}_{\phi,\psi},p^{(n)}_{\tilde{\phi},\tilde{\psi}})\leq 1/100$ for $c$ small enough hence by \cref{lem:2-point-testing-bound}
	\[\inf_{\check{\theta}}\sup_{\theta \in \Theta} \PP_\theta\brackets[\Big]{\abs{\check{p}-p}\vee\abs{\check{q}-q}>c'(\abs{p-\tilde{p}}\vee \abs{q-\tilde{q}})}\geq 1/4.\]
	Inserting from \cref{eqn:max-pq} we conclude the bound in either case.

	For $(f_0,f_1)$, again set $\phi=(1-3\delta,\epsilon,\zeta(1+S)^{1/2})$, $\tilde{\phi}=(1-3\delta-R,\epsilon,\zeta)$ where $R=c\epsilon^{-2}\zeta^{-3}n^{-1/2}$, and choose $\psi=\tilde{\psi}$ by \cref{lem:exists-psi}. As with $p$ and $q$ we deduce that for some $c'>0$ we have
        \begin{multline*}
          \inf_{\check{\theta}}\sup_{\theta\in \Theta} \PP_{\theta}\brackets[\Big]{\norm{\check{f}_0-f_0}\vee\norm{\check{f}_1-f_1}>c'(\norm{f_0-\tilde{f}_0}\vee \norm{f_1-\tilde{f}_1})}\\
          \geq 1/4.
        \end{multline*}
	Using the inversion formulae in \cref{rem:invert-param} and the fact that $\psi=\tilde{\psi}$ and $\norm{\psi_2}=1$, one calculates
	\[ 2(\norm{f_0-\tilde{f}_0}\vee\norm{f_1-\tilde{f}_1}) = \abs{\phi_1\phi_3-\tilde{\phi}_1\tilde{\phi}_3}+\abs{\phi_3-\tilde{\phi}_3}\geq \abs{\phi_3-\tilde{\phi}_3}\]
	For the current choice of $\phi,\tilde{\phi}$, calculating as in proving \cref{thm:mlb} \cref{item:mlb:phi3}, we have $\abs{\phi_3-\tilde{\phi}_3}\geq C\zeta R/\delta$ for some $C>0$ and we deduce the lower bound.
\end{proof}

\begin{proof}[Proof of \cref{cor:learning-rate}]
	It suffices to substitute $\alpha=e^{-x^2}$ into \cref{cor:rates-for-theta} and solve for error equal to $E$, while ensuring that $x^2=\log(1/\alpha)$ is suitably bounded.
\end{proof}

\section*{Acknowledgments}

This work was supported by a public grant as part of the Investissement d'avenir project, reference ANR-11-LABX-0056-LMH, LabEx LMH,  by Institut Universitaire de France, and by the EPSRC Programme Grant on the Mathematics of Deep Learning, under the project: EP/V026259/1.

\bibliographystyle{IEEEtran}
\bibliography{bibliography}



\begin{IEEEbiographynophoto}{Kweku Abraham}
  received his PhD at the University of Cambridge in 2020, where he currently works as a researcher on the EPSRC programme grant Maths4DL, after working at Université Paris-Saclay from 2019-2021 as a Fondation Mathématique Jacques Hadamard postdoctoral fellow. His research interests include guarantees for Bayesian and deep learning methods, and latent variable modelling.
\end{IEEEbiographynophoto}

\begin{IEEEbiographynophoto}{Elisabeth Gassiat}
  took her PhD at the Université Paris-Sud in 1988, and is currently
  Professor in the Mathematics Department at Université Paris-Saclay. Her main
  research interest is in statistical learning, including non-parametric
  statistics, mixture and hidden Markov modeling, bayesian inference, and coding
  theory.
\end{IEEEbiographynophoto}

\begin{IEEEbiographynophoto}{Zacharie Naulet}
  received his PhD in Mathematics and Statistics from the University Paris Dauphine in 2016, and is currently an Assistant Professor in the Mathematics Departement at Université Paris-Saclay. His main research interests are statistical learning, bayesian statistics, network modeling, and mixture modeling.
\end{IEEEbiographynophoto}

\vfill

\end{document}